%% file: fedcurve.tex
\DocumentMetadata{pdfstandard=A-2b, pdfversion=1.7, lang=en}

\documentclass[sigconf, nonacm]{acmart}

\usepackage{pvldb}% VLDB/PVLDB formatting overrides, see pvldb.sty for details.
\renewcommand\vldbdoi{XX.XX/XXX.XX}
\renewcommand\vldbpages{XXX-XXX}
\renewcommand\vldbavailabilityurl{https://github.com/xuefeng-xu/fedcurve}

\usepackage{url}

\usepackage{hyperref} 

\usepackage{amsmath}
\allowdisplaybreaks

\usepackage{graphicx}
\usepackage{subfig}

\usepackage{booktabs}
\usepackage{multirow}

\usepackage{mathtools}
\usepackage{amsthm}

\usepackage{algorithm}
\usepackage[noend]{algpseudocode}

\usepackage{enumitem}

\newlength{\figwidth}
\begin{document}

% \onecolumn
% \pagestyle{empty}
% \thispagestyle{empty}
% \input{response}
% \twocolumn
% \addtocounter{page}{-3}
% \pagestyle{plain}
% \clearpage
% \newpage

\title{Federated Computation of ROC and PR Curves}

%%
%% The "author" command and its associated commands are used to define the authors and their affiliations.
\author{Xuefeng Xu}
\orcid{0000-0002-7886-7558}
\affiliation{%
  \institution{University of Warwick}
  \city{Coventry}
  \country{United Kingdom}
}
\email{xuefeng.xu@warwick.ac.uk}

\author{Graham Cormode}
\orcid{0000-0002-0698-0922}
\affiliation{%
  \institution{University of Oxford}
  \city{Oxford}
  \country{United Kingdom}
}
\email{graham.cormode@cs.ox.ac.uk}

\begin{abstract}
Receiver Operating Characteristic (ROC) and Precision-Recall (PR) curves are fundamental analytical tools for evaluating classification models, providing a comprehensive view of performance across decision thresholds. In modern data management settings, such evaluation must often be performed over data that is distributed across multiple parties and subject to strict privacy constraints. This setting arises naturally in federated learning (FL) and collaborative data platforms, where raw prediction scores and labels cannot be centrally collected.
We study the problem of computing ROC and PR curves over distributed data under formal privacy guarantees. The key challenge is that exact computation requires access to all prediction scores and labels, leading to linear communication costs and privacy risks. To address this, we propose a novel framework for approximating ROC and PR curves using quantiles of the score distribution, which can be computed efficiently under secure aggregation and distributed differential privacy. Our approach requires only $O(Q)$ communication, where $Q$ is the number of quantiles, and avoids sharing raw data entirely.
We provide theoretical guarantees on the approximation quality by bounding the Area Error (AE) between the true and estimated curves, demonstrating principled trade-offs among approximation accuracy, privacy protection, and communication cost. Our method is robust to data heterogeneity and skew, making it suitable for real-world distributed data management scenarios.
Experiments on real-world datasets show that our approach achieves high accuracy with low communication overhead under strong privacy constraints. Overall, this work provides a practical and theoretically grounded solution for privacy-preserving evaluation of classification models over distributed data.
\end{abstract}

\maketitle

\vldbtopmatter

\section{Introduction}
\label{sec:intro}
Modern data management systems increasingly operate in decentralized environments where data is distributed across multiple parties and cannot be directly shared due to privacy constraints. This paradigm appears in federated learning (FL)~\citep{Kairouz2021a}, cross-silo analytics, and collaborative data platforms. In such settings, a central coordinator must support global analytics and monitoring tasks without direct access to raw data, raising fundamental challenges in distributed query processing under privacy constraints.
Developing effective solutions for managing these distributed workloads requires novel systems and algorithmic solutions.

A critical yet underexplored task in this setting is the evaluation of classification models. In practice, model evaluation is not a one-time procedure: it is required during training for monitoring convergence; before deployment for model selection; and after deployment for continuous performance tracking. While existing FL systems typically support simple aggregate metrics such as accuracy or loss (e.g., Flower~\citep{Beutel2020}), these metrics are insufficient for understanding model behavior. Accuracy, in particular, is unreliable on imbalanced datasets \citep{Provost1998}, where a classifier may appear effective simply by predicting the majority class. Loss values are often not interpretable in practice and provide limited insight for deployment decisions. Moreover, one model may achieve a lower loss value, but have worse ROC/PR curves compared to another. Without computing these curves in the federated setting, such critical differences in model behavior would remain hidden. This fundamental limitation prevents a full assessment of model performance, hindering informed decisions about model deployment and improvement.

In contrast, Receiver Operating Characteristic (ROC) \citep{Provost1998} and Precision-Recall (PR) \citep{Raghavan1989} curves are standard analytical tools that characterize model performance across all decision thresholds. These curves provide richer insights into trade-offs such as true positive versus false positive rates, and precision versus recall. However, computing these curves in a distributed and privacy-preserving setting poses significant challenges. First, exact computation of ROC and PR curves requires access to all prediction scores and labels, which conflicts with the fundamental constraints of distributed data management systems where raw data cannot be centralized. Second, even if data sharing were allowed, the communication cost of transmitting all scores scales linearly with the dataset size, making it impractical for large-scale systems. Third, prior work has shown that even aggregate outputs such as ROC curves may leak sensitive information~\citep{Matthews2013}, and access to prediction scores can enable membership inference attacks~\citep{Shokri2017}. These challenges highlight the need for principled methods that support complex analytical queries over distributed data while providing strong privacy guarantees.
We therefore seek solutions that are accurate, private, computationally and communication-efficient, suitable for implementation in emerging data management systems.

In this paper, we formulate ROC and PR curve computation as a distributed analytics problem and propose a novel approximation framework based on quantiles of the prediction score distribution. 
Our methods combine techniques from data privacy~\cite{Dwork2014} and data sketching~\cite{Liberty2012}.
Our key insight is that ROC and PR curves can be approximated using a small number of quantile points, which can be computed efficiently in a distributed manner without exposing raw data. We design protocols that compute these quantiles using histogram aggregation, enabling accurate and efficient federated estimation regardless of data heterogeneity. Each client locally bins its prediction scores and sends histograms to the server, which reconstructs the global score distribution without accessing raw data. For privacy protection, we integrate distributed differential privacy (DDP)~\citep{Goryczka2017}, where each client adds independent noise to their local histograms before aggregation. The aggregated noisy histograms are then used to extract quantiles and construct approximate ROC and PR curves via monotone interpolation. 
This design provides end-to-end privacy guarantees while maintaining communication efficiency and robustness to heterogeneous client data distributions.

We analyze the accuracy of our approach using the Area Error (AE) (Definition \ref{def:ae}), which measures the deviation between the true and approximate curves. Under mild assumptions on the score distribution, i.e., that they are $O(1)$-Lipschitz (Definition~\ref{def:well-behaved}), we show that the AE is $O(1/Q)$ for the ROC curve, where $Q$ is the number of quantiles, which governs the communication cost. The AE bound is $\tilde{O}(1/Q)$~\footnote{As usual, $\tilde{O}$ notation suppresses logarithmic factors.} for the PR curve under mild class imbalance, where the class ratio $r=n^+/n^-\ge 0.1$, with $n^+$ and $n^-$ representing the number of positive and negative examples, respectively. 
Under extreme imbalance, i.e., $r\ll1$, the AE is bounded by $\tilde{O}(1/(Qr))$. This bound is extended to $\tilde{O}(1/Q + 1/(n^+\varepsilon) + 1/(n^-\varepsilon))$~\footnote{In the extreme class imbalance case for the PR curve, the error bound becomes $\tilde{O}(1/(Qr) + 1/(n\varepsilon))$; throughout the paper, our analysis focuses on mild class imbalance, where the error is $\tilde{O}(1/Q)$.} under distributed differential privacy \citep{Dwork2014}, where $n^+$ and $n^-$ are the number of positive and negative examples and $\varepsilon$ is the standard privacy budget. Notably, our method is agnostic to data heterogeneity and skew across clients, and requires only $O(Q)$ communication, providing an efficient accuracy-privacy-communication trade-off.

Empirically, we show that with $Q\approx 10^2$, the AE of ROC is often close to $10^{-3}$, the AE of PR is below $10^{-2}$, and these remain low even under strong privacy settings (e.g., $\varepsilon \le 1$). While the well-behavedness condition is instrumental in proving the theoretical error bounds, our experiments demonstrate that the method achieves consistently low area error even when this condition is violated in practice (see the score distribution in Figure~\ref{fig:ecdf-xgb}), indicating robustness across diverse datasets and classifiers.
By implementing our method within an existing distributed system (Flower), we confirm that the communication cost is small in practice, and that other computational costs (time and memory usage) are also low.

Our work bridges techniques from distributed data management, approximate query processing, and differential privacy to enable a new class of analytical queries, i.e., performance evaluation of trade-off curves, over sensitive, distributed data. We present and evaluate novel algorithms for providing characterizations of classifier efficacy with guaranteed privacy and accuracy. The proposed approach is simple to implement, communication-efficient, and compatible with existing federated data-processing systems.

\smallskip
\noindent
\textbf{Contributions.} 
Our main contributions are as follows:
\begin{enumerate}[leftmargin=*]
\item We formulate the computation of ROC and PR curves through distributed queries over the prediction score distributions. 
We show how both curves can be reconstructed from compact approximations of the distributions.
\item We propose a quantile-based approach for approximating the empirical cumulative distribution functions (ECDFs) and construct the ROC and PR curves in the federated setting with $O(Q)$ communication and no raw data sharing, independent of the number of examples and data heterogeneity across clients.
\item We introduce Area Error as a measure of full-curve approximation and establish an $O(1/Q)$ bound for the ROC curve and $\tilde{O}(1/Q)$ for the PR curve under a mild smoothness condition of score distribution.
\item We present a privacy-preserving variant of our protocol and extend the analysis of the Area Error to secure aggregation and distributed differential privacy. The resulting bounds explicitly capture the trade-offs between accuracy, privacy, and communication cost.
\item We perform empirical validation showing high approximation accuracy and low communication overhead under strong privacy guarantees. We implement the protocol in Flower using SecAgg+ protocol~\citep{Bonawitz2017} and demonstrate its practical feasibility.
\end{enumerate}
While several components of our solution are standard techniques in the literature (secure aggregation, hierarchical histograms, private noise addition), the resulting protocol for ROC and PR curve materialization is novel. 
In particular, the error analysis in Section~\ref{sec:ae-analysis} provides the first accuracy guarantee for this problem.

\smallskip
\noindent
\textbf{Outline.}
Section~\ref{sec:related-work} discusses related work.
Section~\ref{sec:preliminaries} reviews the relevant background. Section~\ref{sec:fed-roc-pr} describes our method for federated computation of ROC and PR curves via quantiles. Section~\ref{sec:ae-analysis} presents theoretical bounds on the Area Error under both exact and approximate quantiles, as well as under privacy constraints. Section~\ref{sec:evaluation} reports empirical validation on multiple datasets and compares against the prior work. Section~\ref{sec:discussion} discusses further extensions and limitations, and Section~\ref{sec:conclusion} concludes the paper.
\section{Related Work}
\label{sec:related-work}
Several problems related to evaluation of classifier performance have been studied across data management and machine learning in recent years.
\citet{Fawcett2006} proposed heuristic methods for merging ROC curves across split datasets, either by averaging true positive rates (TPRs) at fixed false positive rates (FPRs), or by averaging FPR-TPR pairs at fixed thresholds.
However, these methods lack any formal error guarantees.
Moreover, these averaging schemes do not directly recover the global ROC curve under heterogeneous client data.
\citet{Chen2016} investigated differentially private ROC computation in centralized settings using median threshold selection and range queries, followed by post-processing for monotonicity, but did not provide formal error bounds. 
It is also unclear how these methods could be adapted to a distributed setting.

More recent work \citep{Bell2020} addresses AUC computation, which is a scalar summary of the ROC curve, under Local DP using hierarchical histograms.
The error bound was improved by \citet{Cormode2023} using similar measurements, with relaxed assumptions.
However, the magnitude of error due to local DP is large, and so we seek a better tradeoff between privacy and accuracy.
\citet{Sun2023} explored DP-based AUC computation in a distributed setting. 
Their approach is to compute the FPRs and TPRs locally at each client, add noise and then combine them.
While this guarantees privacy, it does not provide any bound on the result accuracy.

While AUC provides a scalar performance summary, we argue that the full ROC and PR curves offer a more detailed and hence valuable view, requiring different techniques for computation and error analysis. 
In parallel with the development of this work, \citet{Barczewski2025} recently proposed differentially private computation of ROC curve in FL.
Their proposed approach implements a uniformly spaced points that (approximately) define the ROC curve. 
For privacy and security, the algorithm adds differentially private noise to the gathered statistics, and uses function secret sharing to securely evaluated them.
Finally, post-processing is applied to smooth the noisy curves produced. 
Their analysis bounds the \textit{Squared Point Error} at the selected query thresholds. Because their guarantee and ours differ in the error metric, evaluation domain, and underlying assumptions, they are not directly comparable. Their approach also incurs greater computational and communication overhead, while our empirical evaluation shows that our method achieves lower Area Error in most tested settings.
\section{Preliminaries}
\label{sec:preliminaries}
This section reviews the key concepts underlying our approach, including ROC/PR curves, differential privacy in distributed settings, federated quantile estimation, and monotone interpolation.
\subsection{The ROC and PR Curves}
\label{sec:roc-pr-basic}
The \textbf{ROC Curve} evaluates classifiers by plotting the True Positive Rate (TPR) against the False Positive Rate (FPR). The x-axis represents the FPR, defined as $F = \text{FP}/n^{-}$, where FP counts false positives (as the score threshold varies) and $n^{-}$ is the total number of negative examples. 
The y-axis represents the TPR, defined as $T = \text{TP}/n^{+}$, where TP counts true positives and $n^{+}$ is the number of positive examples. Since both FPR and TPR are independent of the class distribution, the ROC curve remains invariant to class skew.
An efficient algorithm for constructing the ROC curve is outlined by \citet{Fawcett2006}. It involves sorting prediction scores in descending order, using each score as a threshold, and computing the corresponding FPRs and TPRs. Linear interpolation is applied between points. The ROC curve is non-decreasing, progressing from (0,0) (score threshold $s=1$) to (1,1) ($s=0$). A random classifier lies along the diagonal $y = x$, while curves bending toward the top-left indicate better performance. The area under the ROC curve (AUC-ROC) is a common summary metric, computed via the trapezoidal rule.

\smallskip
\noindent
The \textbf{PR Curve} plots precision against recall. The x-axis represents recall, equals TPR, while the y-axis represents precision, defined as $P = \text{TP}/(\text{TP+FP})$. Unlike the ROC curve, the PR curve is sensitive to class imbalance since precision depends on both classes. This property makes the PR curve more informative when positive instances are rare \citep{Saito2015}. The PR curve is non-monotonic, spanning from (0,1) ($s=1$) to (1, $n^{+}/n$) ($s=0$), as precision may decrease with increasing recall. A random classifier yields a horizontal line at $y = n^{+}/n$, while curves bending toward the top-right are better.

There exists a one-to-one correspondence between ROC and PR curves \citep{Davis2006}. In ROC space, linear interpolation is the standard method for connecting points. However, in PR space, linear interpolation can lead to overestimation of performance because precision is a non-linear function of TPR and FPR. In \textit{Scikit-learn}~\citep{Pedregosa2011}, right-end constant interpolation is used to ensure consistency with average precision, which is a scalar summary of the PR curve and differs from the trapezoidal rule used for AUC-ROC.
\subsection{Differential Privacy and Federated Quantiles}
\label{sec:dp-fed-quantiles}
\textbf{Differential Privacy (DP)~\cite{Dwork2014}} guarantees that the output of a function $f$ is insensitive to the inclusion or exclusion of any single data point. Formally, a mechanism $\mathcal{M}$ satisfies $\varepsilon$-DP if for any two neighboring datasets $D$ and $D'$ differing by one record, the output distributions $\mathcal{M}(D)$ and $\mathcal{M}(D')$ are close, with probability ratio bounded by $e^{\varepsilon}$, i.e., $\Pr[\mathcal{M}(D) \in S] \leq e^{\varepsilon} \Pr[\mathcal{M}(D') \in S]$ for all subsets $S$ of the output space. 
A smaller $\varepsilon$ implies stronger privacy but requires more noise, which may lead to less accurate results. The noise magnitude also depends on the sensitivity of the function, defined as $\Delta f = \max_{D,D'} \|f(D)-f(D')\|$. One common way to achieve DP is to add Laplace noise to the output of $f$, where the noise scale is proportional to $\Delta f/\varepsilon$, e.g., $\text{Lap}(0, \Delta f/\varepsilon)$.
When an $\varepsilon$-DP mechanism is applied $T$ times on overlapping inputs, \textit{composition bounds} show that the privacy guarantee varies as $T\varepsilon$.

\smallskip
\noindent
\textbf{Secure Aggregation (SA)} \citep{Bonawitz2017} is a cryptographic protocol that enables a server to compute the aggregate (sum) of client inputs without revealing any individual contribution. 
Each client masks its local input using random values that cancel out only after aggregation. 
The server receives only the masked inputs and can recover the total sum, while learning nothing about any client's individual data beyond what is implied by the final aggregate. 
Modern SA implementations can tolerate client dropouts and communication failures, making them practical for large-scale federated systems.
SA has become a standard building block for federated learning and privacy-preserving distributed analytics, where many computations, such as vector sums, averages, histograms, and other aggregate statistics, can be expressed as additive operations.

\smallskip
\noindent
\textbf{Distributed DP (DDP)} \citep{Goryczka2017} allows clients to collaboratively add noise to ensure privacy. For example, each client adds local noise for which the scale is proportional to $1/m$ of the total noise required in the centralized setting, where $m$ is the number of clients. Therefore, after aggregating noise from all clients, the total noise scale is the same as in the centralized setting. To achieve $\varepsilon$-DDP, continuous noise can be sampled from the difference of two independent Gamma distributions, while discrete noise can be sampled from the difference of two independent P\'{o}lya distributions. This also extends to $(\varepsilon,\delta)$-DP mechanisms such as the Skellam Mechanism \citep{Agarwal2021} and the Distributed Discrete Gaussian Mechanism \citep{Kairouz2021}. Generally, DDP can be achieved with Secure Aggregation \citep{Bonawitz2017}.

\smallskip
\noindent
\textbf{Federated Quantile Estimation} can be achieved using histogram-based methods. Each client bins local data and transmits a noisy histogram. The server aggregates these to form a global histogram from which quantiles are estimated. 
Hierarchical histograms \citep{Hay2010, Qardaji2013, Cormode2019} are preferred over flat histograms, as they provide improved accuracy and enable consistency enforcement via post-processing, making them well-suited for federated and differentially private settings. Specifically, a hierarchical histogram of height $h$ divides the score range into $b^i$ bins at level $i$, where $b$ is the branching factor. Each client splits its privacy budget $\varepsilon$ equally across the $h$ layers and adds independent noise to each bin. The server aggregates histograms via simple summation.
While there are other approaches to DP quantile estimation (e.g., \cite{Gillenwater2021}), these are not well-suited to the federated setting; hence we adopt the histogram approach.
\subsection{Monotone Interpolation}
\label{sec:monotone-interp}
Linear interpolation is a simple and efficient method for joining up a monotone series of points but may lack smoothness and fail to capture underlying trends of the data. Polynomial interpolation improves smoothness but may introduce oscillations, potentially violating monotonicity. 
Piecewise-cubic Hermite interpolation (PCHIP) \citep{Fritsch1980, Fritsch1984, Moler2004} provides a balance: it constructs piecewise cubics that preserve monotonicity and avoid overshoot via carefully selected derivatives at each point. Although higher-order methods like quintic interpolation \citep{Costantini1997, Lux2023} can offer smoother results, they are less efficient. 
In this work, we make use of both linear and PCHIP interpolation for reconstructing curves from quantiles.
\section{Federated ROC and PR Curves}
\label{sec:fed-roc-pr}
In this section, we present our method for approximating ROC and PR curves in the federated setting. 
We first describe how to construct curves from score quantiles and then explain how to estimate those quantiles using histograms under both SA and DDP settings.
We assume that all parties are honest-but-curious; that is, the server follows the protocol but may attempt to infer sensitive information from the received data from clients.
\subsection{Curve Approximation via Quantiles}
\label{sec:curve-approx}
Constructing ROC and PR curves uses prediction scores and corresponding class labels. In federated settings, these raw values cannot be shared due to privacy constraints. Instead, since ROC and PR curves are determined by the distributions of scores for positive and negative classes, we will estimate the empirical cumulative distribution function (ECDF) $\Phi(s)$ of prediction scores using quantiles, computed separately for the positive and negative classes. 
Figure~\ref{fig:cdf-quantile} shows a concrete example where the ECDF is divided into five evenly spaced quantile intervals, each with a width of 0.2. PCHIP interpolation is used to estimate the ECDF between quantile points. 
We can then compute TPR/FPR (for ROC) and Precision/Recall (for PR) at arbitrary thresholds using the interpolated ECDFs.
\begin{figure}[t]
\centering%
\includegraphics[width=0.5\columnwidth]{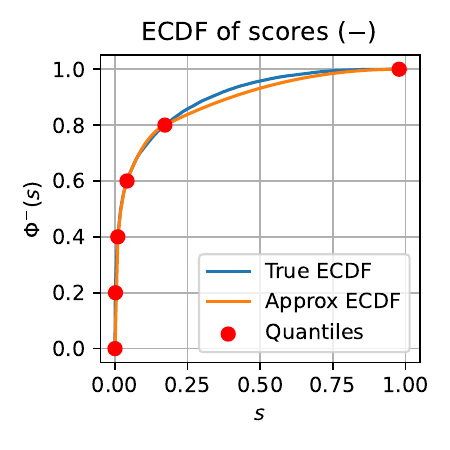}%
\label{fig:cdf-quantile-neg}%
\includegraphics[width=0.5\columnwidth]{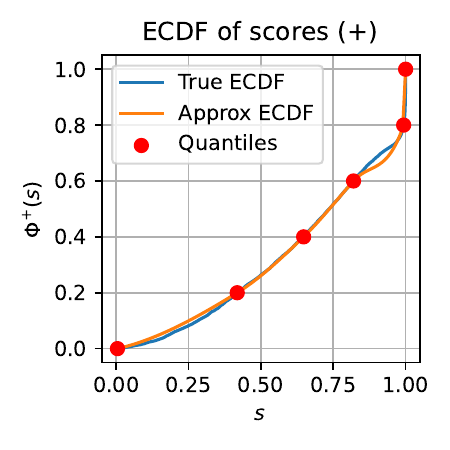}%
\label{fig:cdf-quantile-pos}%
\caption{ECDF approximation using six quantiles for prediction scores of negative and positive classes.}%
\label{fig:cdf-quantile}%
\Description[]{}%
\end{figure}

For the ROC curve, the FPR depends only on the negative score distribution, while TPR depends only on the positive score distribution.
They are computed as:
\begin{align}
\label{eq:fpr-tpr}
F(s)&=\Pr[S > s\mid y=0]=1 - \Phi^{-}(s),\\
T(s)&=\Pr[S > s\mid y=1]=1 - \Phi^{+}(s).
\end{align}
For the PR curve, the recall equals $T(s)$, and precision $P(s)$ is computed using $T(s)$ and $F(s)$ as:
\begin{equation}
\label{eq:precision}
P(s) = \frac{T(s)n^{+}}{T(s)n^{+} + F(s)n^{-}}.
\end{equation}
Figure~\ref{fig:curve-approx} shows the resulting approximated ROC and PR curves derived from the estimated ECDFs. 
In this example, the classifier used is XGBoost \citep{Chen2016a} on the Adult dataset \citep{Becker1996}, although our approach is agnostic to the classifier and data used. 
For exposition, we assume that prediction scores lie in $[0,1]$ (and so may be thought of as probabilities), but the approach applies to arbitrary score ranges.
\begin{figure}[t]
\centering%
\includegraphics[width=0.5\columnwidth]{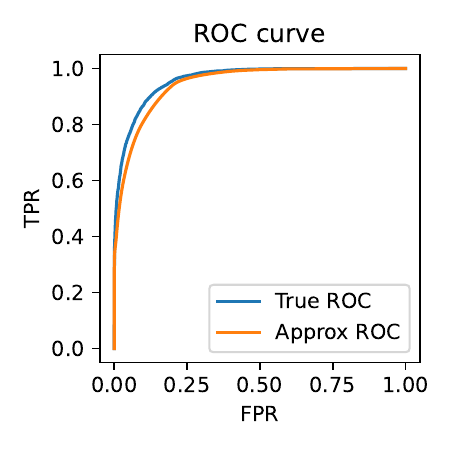}%
\label{fig:roc-approx}%
\includegraphics[width=0.5\columnwidth]{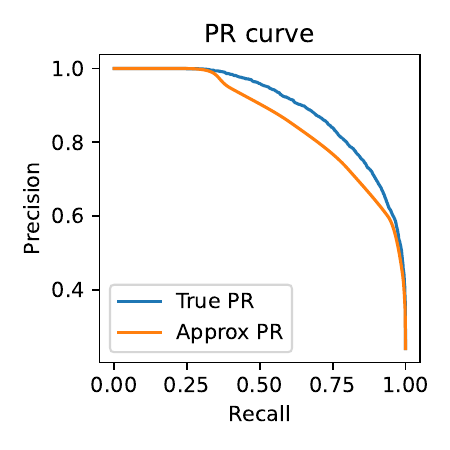}%
\label{fig:pr-approx}%
\caption{Approximating ROC and PR curves constructed from approximated ECDFs.}%
\label{fig:curve-approx}%
\Description[]{}%
\end{figure}
\subsection{Quantile Estimation via Histograms}
\label{sec:quantile-histogram}
Quantiles are estimated via hierarchical histograms of height $h$. At each level $i$ ($1 \leq i \leq h$), the prediction score range is divided into $b^i$ equal-width bins, where $b$ is the branching factor. Each client builds separate histograms for positive and negative examples and sends them to the server. Under DP, each client splits its privacy budget $\varepsilon$ equally across the $h$ layers and adds independent noise to each bin \citep{Hay2010, Qardaji2013}. The histograms are aggregated at the server, enabling quantile estimation without accessing raw scores.
The server uses this information to compute a (noisy) ECDF of the score distribution, and then derives the FPR and TPR values to plot the ROC and PR curves.
The full protocol is described in Algorithm~\ref{alg:fed-roc-pr}.
Highlighting indicates communication between clients and the server, i.e., the histograms aggregated via secure aggregation.
\begin{algorithm}[t]
\caption{Federated Construction of ROC and PR Curves}
\label{alg:fed-roc-pr}
\begin{algorithmic}[1]
\Statex \textbf{Client Part:}
\State \textbf{Input:} Prediction scores for positives $s^+$ and negatives $s^-$
\State Build histograms $h^+,\ h^-$ over $s^+,\ s^-$
\If{Differential Privacy is enabled}
  \State Add i.i.d. privacy noise to every histogram bin
\EndIf
\State \colorbox{orange!20}{Send histograms $h^+$ and $h^-$ to the server}
\setcounter{ALG@line}{0}

\Statex
\Statex \textbf{Server Part:}
\State \colorbox{orange!20}{Aggregate client histograms to obtain global $H^+$ and $H^-$}
\If{Differential Privacy is enabled}
  \State Apply post-processing to noisy $H^{+}$ and $H^{-}$
\EndIf
\State Extract total positive/negative counts $\hat{n}^+,\ \hat{n}^-$ from $H^+,\ H^-$
\State Extract a set of quantiles $\hat{q}^{+},\ \hat{q}^{-}$ from $H^+,\ H^-$
\State Interpolate $\hat{q}^{+},\ \hat{q}^{-}$ to build ECDFs $\hat{\Phi}^+(s),\ \hat{\Phi}^-(s)$
\For{each evaluation threshold $s$}
  \State $\hat{F}(s) \gets 1-\hat{\Phi}^{-}(s)$
  \Comment{FPR}
  \State $\hat{T}(s) \gets 1-\hat{\Phi}^{+}(s)$
  \Comment{TPR / Recall}
  \State $\hat{P}(s) \gets \dfrac{\hat{T}(s)\,\hat{n}^{+}}{\hat{T}(s)\,\hat{n}^{+} + \hat{F}(s)\,\hat{n}^{-}}$ \Comment{Precision}
\EndFor
\State Plot ROC curve: $(\hat{F}(s),\hat{T}(s))$; Plot PR curve: $(\hat{T}(s),\hat{P}(s))$
\end{algorithmic}
\end{algorithm}

\smallskip
\noindent
\textbf{Privacy Guarantee.}
Due to the addition of suitably scaled private noise, the histograms received by the server achieve $\varepsilon$-differential privacy (Section~\ref{sec:dp-fed-quantiles}). 
Moreover, since this is the only access to sensitive client data, the resulting algorithm meets the $\varepsilon$-DP guarantee.
If applied repeatedly $T$ times, the guarantee is $T\varepsilon$-DP (Section~\ref{sec:dp-fed-quantiles}).

\smallskip
\noindent
\textbf{Data Heterogeneity.}
A positive property of the histogram-based quantile estimation is that it is agnostic to data heterogeneity and skew. Since each client locally bins scores and the server aggregates counts additively, the resulting global histogram matches the centralized case, regardless of how data is distributed across clients.

\smallskip
\noindent
\textbf{Client Dropout.}
The secure aggregation protocol is designed to tolerate client dropout, as the server can still compute the aggregated result provided the number of participating clients is above a certain threshold. 
This threshold represents the minimum number of surviving clients needed to reconstruct the secret masks.

\smallskip
\noindent
\textbf{Communication Cost.}
The communication cost is linear in the number of leaf bins: $O(b^h)$. Under SA, each client sends only the $b^h$ leaf counts. Under DP, a client sends $b + b^2 + \cdots + b^h = \frac{b(b^h - 1)}{b - 1}$ bins, which is also $O(b^h)$. To ensure accurate quantile estimation, we require $b^h > Q$, where $Q$ is the number of quantiles. To balance communication and accuracy, we set $h = \lceil \log_b Q \rceil + c$, where $c$ is a small constant (typically 2 or 3). This keeps the communication cost at $O(Q)$, which is much smaller than the dataset size. In our experiments, we use $b=2$ and $c=2$. 
When $Q=1024$, this is sufficient to ensure low Area Error (Definition~\ref{def:ae}). Under this setting, each client transmits approximately $2^{13}$ bins per histogram, totaling around 8K integers (32-bit), i.e., 32K Bytes.
Without privacy constraints, there is only a single communication round; using Secure Aggregation \citep{Bonawitz2017,Bell2020a}, this increases to four rounds or fewer.
\section{Analysis of Area Error}
\label{sec:ae-analysis}
In this section, we define the Area Error (AE) metric for quantifying the discrepancy between true and approximated ROC and PR curves, and provide theoretical bounds on AE when using quantiles for curve approximation. We first analyze the case of exact quantiles, and then extend the analysis to the federated setting with approximate quantiles and privacy noise.

\smallskip
\noindent
\textbf{Shorthand Notation.}
Throughout this section, we adopt the following shorthand notation. For function composition involving $T$, $F$ and their inverses, we write $T(f) := T(F^{-1}(f))$ to denote TPR as a function of FPR $f$, and $F(t) := F(T^{-1}(t))$ to denote FPR as a function of TPR $t$. Similarly, $\hat{T}(f) := \hat{T}(\hat{F}^{-1}(f))$ and $\hat{F}(t) := \hat{F}(\hat{T}^{-1}(t))$ denote their estimated counterparts. Other mixed compositions, such as $T(\hat{F}^{-1}(f))$ or $\hat{T}(F^{-1}(f))$, are stated explicitly when needed. For PR curves, we use $P(t) := P(T^{-1}(t))$ for precision as a function of recall, and $\hat{P}(t) := \hat{P}(\hat{T}^{-1}(t))$ for the estimated precision.
\subsection{Area Error Definition}
We use Area Error (AE) to quantify the discrepancy between the true and approximated ROC or PR curves, similar in spirit to the notion of the $L_1$ metric from \citet{Clemencon2009}. Formally:
\begin{definition}[AE]
\label{def:ae}
The Area Error is defined as the integral of the absolute difference between the true and estimated curves over the full domain, i.e., FPR for ROC and recall for PR from 0 to 1:
\begin{equation}
\text{AE}_\text{ROC} = \int_0^1 |T(f) - \hat{T}(f)| df, \quad
\text{AE}_\text{PR} = \int_0^1 |P(t) - \hat{P}(t)| dt.
\end{equation}
\end{definition}
Here, $T(f)$ and $\hat{T}(f)$ denote the true and estimated TPR at FPR $f$, while $P(t)$ and $\hat{P}(t)$ denote the true and estimated precision at recall $t$ (since recall is equivalent to TPR).

AE is a stricter error measure than absolute AUC difference (which computes $\left|\int_0^1 T(f)df - \int_0^1 \hat{T}(f)df\right|$ for ROC). A small AE implies a small AUC error, but not vice versa. This is because AUC aggregates performance into a single scalar, while AE captures curve discrepancies over the full domain. 
Another weak alternative is the Point Error~\citep{Barczewski2025}, which averages the absolute or squared difference at the score thresholds. However, this neglects cumulative errors and lacks robustness across the full curve.
\subsection{Bounding Area Error Using Exact Quantiles}
\label{sec:ae-bounds-exact-q}
We first analyze the case where \textbf{exact} quantiles are used for curve approximation. This serves as a theoretical baseline for understanding the error introduced by quantile-based approximation, without the additional complications of estimation error and privacy noise.
We consider the case where quantiles are evenly spaced. Let $f$ and $t$ denote the exact FPR and TPR at arbitrary score threshold $s$, and let $\hat{f}$ and $\hat{t}$ be their estimations. If $Q^-$ and $Q^+$ exact quantiles are used for the negative and positive classes respectively, then:
\begin{equation}
|f-\hat{f}| \le \frac{1}{2(Q^--1)}, \quad |t-\hat{t}| \le\frac{1}{2(Q^+-1)}.
\end{equation}
The bound arises because the maximum quantile error corresponds to half the width of a quantile interval. Using this, we can now bound the AE for the ROC curve.
\begin{theorem}[ROC-AE-ExactQ]
\label{thm:ae-roc-exactq}
If $Q$ \textbf{exact} quantiles are used for both positive and negative examples, then the Area Error between the true and estimated ROC curves is bounded by $O(1/Q)$.
\end{theorem}
\begin{proof}[Proof]
\label{proof-ae-roc-exact-q}
The high-level idea is that when querying the distribution of the negative class at the exact quantile points, the false positive rates (FPRs) are found precisely; when probing elsewhere in the distribution, the FPR absolute error is guaranteed to be at most half the quantile width (the distance between successive quantile measurements). The same holds for true positive rates (TPRs) at quantile points for the positive class. As a result, we evenly divide the ROC space into $Q-1$ vertical strips, where each strip contributes an error of at most $O(1/Q^2 + \Delta_i/Q)$, with the first term accounting for the bounded deviation in TPR/FPR and the second term corresponding to a region of uncertain area. Since the ROC space is bounded and $\sum_i \Delta_i\le 1$, the overall Area Error is bounded by $O(1/Q)$ when summing over all $Q-1$ strips. In the remainder of this proof, we formally establish this result.

Since TPRs and FPRs are computed independently from separate class distributions, we consider two cases, expanded below: (1) using $Q^-$ exact quantiles for negative examples to precisely determine FPRs, while TPRs remain approximated; 
(2) additionally using $Q^+$ exact quantiles for positive examples to tighten the TPR estimates. 
To formalize this bound, we analyze worst-case error by considering how TPR and FPR may deviate between quantile points.
\subsubsection*{(1) Using \texorpdfstring{$Q^-$}{Q-} Quantiles for Negative Examples}
\label{sec:negquant}
Dividing the FPR axis using $Q^-$ quantiles results in $Q^- - 1$ vertical strips. For any pair of consecutive negative quantiles, $q^-_i$ and $q^-_{i+1}$, the corresponding FPR values $f_i=F(q^-_i)$ and $f_{i+1}=F(q^-_{i+1})$ are accurately estimated. Let $t_i=T(q^-_i)$ and $\hat{t}_i=\hat{T}(q^-_i)$ denote the true and estimated TPRs, respectively. 
In order to show a bound on the AE, we consider the (deterministic) worst possible behavior of the ROC curve.

To maximize AE under the constraints of the measured quantiles, we can set $t_i - \hat{t}_i = t_{i+1} - \hat{t}_{i+1} = \frac{1}{2(Q^+ - 1)}$, the largest possible gap while meeting the quantile definition. 
Define $\Delta_i = t_{i+1} - t_i$ as the change in TPRs between consecutive points. 
This scenario is illustrated in the left panel of Figure~\ref{fig:neg-q-error}: it corresponds to the TPR suddenly ``leaping'' up to $t_{i+1}$ at $f_i$ and remaining constant until $f_{i+1}$ (blue line in the figure). 
The opposite extreme, the TPR remaining steady from $f_i$ to $f_{i+1}$ then leaping up to $t_{i+1}$, would also contribute a symmetric (negative) area error (similar to the dashed grey line), so it suffices to focus on the case of positive area error. 
\begin{figure}[t]
\centering%
\includegraphics[width=0.5\columnwidth]{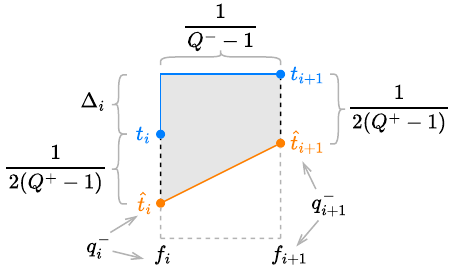}%
\quad%
\includegraphics[width=0.45\columnwidth]{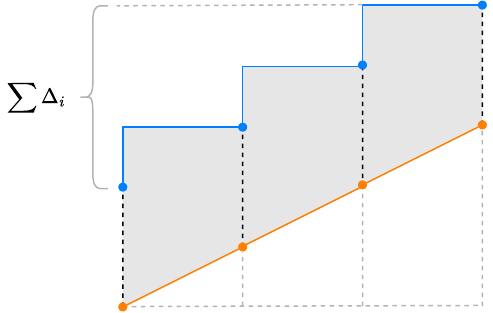}%
\caption{Worst-case area error with quantiles on negative examples only.}%
\label{fig:neg-q-error}%
\Description[]{}%
\end{figure}

The total AE is computed as the sum of the areas of trapezoids, as illustrated in the right panel of Figure~\ref{fig:neg-q-error}. 
We apply the standard trapezoid area formula, $\frac12(a + b)h$, where $a$ and $b$ are the lengths of parallel sides, and $h$ is the (perpendicular) `height'. 
Since $\sum_i \Delta_i \le 1$, the worst-case error is bounded by $O(1/Q)$:
\begin{align}
\text{AE}_\text{ROC}
&\le\sum_{i=1}^{Q^--1} \frac{\frac{1}{2(Q^+-1)}+\frac{1}{2(Q^+-1)}+\Delta_i}{2}\cdot\frac{1}{Q^--1} \\
&=\sum_{i=1}^{Q^--1} \frac{1}{2(Q^+-1)(Q^--1)} + \frac{\Delta_i}{2(Q^--1)}\\
&\le\frac{1}{2(Q^+-1)} + \frac{1}{2(Q^--1)}=O(1/Q)
\end{align}
The proof for the case of (2) adding \texorpdfstring{$Q^+$}{Q+} quantiles for positive examples is analogous.
It respects the same bound but is more complicated; we provide the full details in the Appendix~\ref{sec:posquant}.
\end{proof}
The corresponding area error bound for the PR curve given \textbf{exact} quantiles follows a similar outline, but has more details due to the non-linear relationship between precision and TPR/FPR.
Moreover, since obtaining exact quantiles is often impractical in practice, we focus our study on reporting PR curves using approximate quantiles in the federated setting, which is more relevant for real-world applications.
In Section~\ref{sec:fed-lemma}, we introduce the notion of well-behaved score distributions and establish lemmas that bound the error in TPR and FPR. These lemmas are crucial for analyzing the Area Error in the federated setting with approximate quantiles and privacy noise, which we discuss in Section~\ref{sec:ae-approx-q} and Section~\ref{sec:proof-ae-federated}.
\subsection{Well-Behaved Score Distributions and Lemmas for Approximate Quantiles}
\label{sec:fed-lemma}
In the federated setting, approximate quantiles introduce quantile estimation errors bounded by the width of the finest histogram bin. A key challenge is that these errors can propagate to large errors in TPR/FPR when score distributions may exhibit pathological behavior (e.g., spikes where many examples share the same score). To address this, 
we consider a Lipschitz smoothness condition on score distributions to control how quickly ECDFs may change. This condition, formalized below as ``well-behaved'' score distributions, is similar to assumptions in prior work \citep{Cormode2023} and lets us bound Area Error by relating quantile estimation errors to errors in $T$ and $F$.
\begin{definition}[Well-behaved score distribution]
\label{def:well-behaved}
Let $\Phi^+(s)$ and $\Phi^-(s)$ denote the cumulative distribution functions of the score distribution for positive and negative examples, respectively. We say the distributions are \textit{$\ell$-well-behaved} if both are $\ell$-Lipschitz:
\begin{equation}
|\Phi^+(s_1)-\Phi^+(s_2)|\le\ell|s_1-s_2|\; \mathrm{and}\;
|\Phi^-(s_1)-\Phi^-(s_2)|\le\ell|s_1-s_2|
\end{equation}
\end{definition}
With well-behaved score distributions, we can bound the absolute error of the TPR and FPR using the quantile estimation error $\alpha$. 
These bounds allow us to prove the Area Error of the ROC and PR curves in the federated setting. 
The following lemmas formalize these bounds.
Lemma~\ref{lem:abs-error-score} bounds the error in TPR and FPR at a given score $s$, which is essentially the ECDFs of the score distributions, since the variable is the score $s$;
Lemma~\ref{lem:abs-error-curve} bounds the error in TPR at a given FPR $f$ or FPR at a given TPR $t$, which is more directly relevant for analyzing the Area Error of the curves. 
In particular, $T(f)=T(F^{-1}(f))$ yields the ROC curve, and the variable is the FPR $f$.
This bound is used in Theorem~\ref{thm:ae-roc-approx-q}, the Area Error analysis of the ROC curve. Although $F(t)=F(T^{-1}(t))$ is not the PR curve, it is used in Theorem~\ref{thm:ae-pr-approx-q}, the Area Error analysis of the PR curve, and the variable is the TPR $t$.
\begin{figure*}[t]
\centering%
\includegraphics[width=\figwidth]{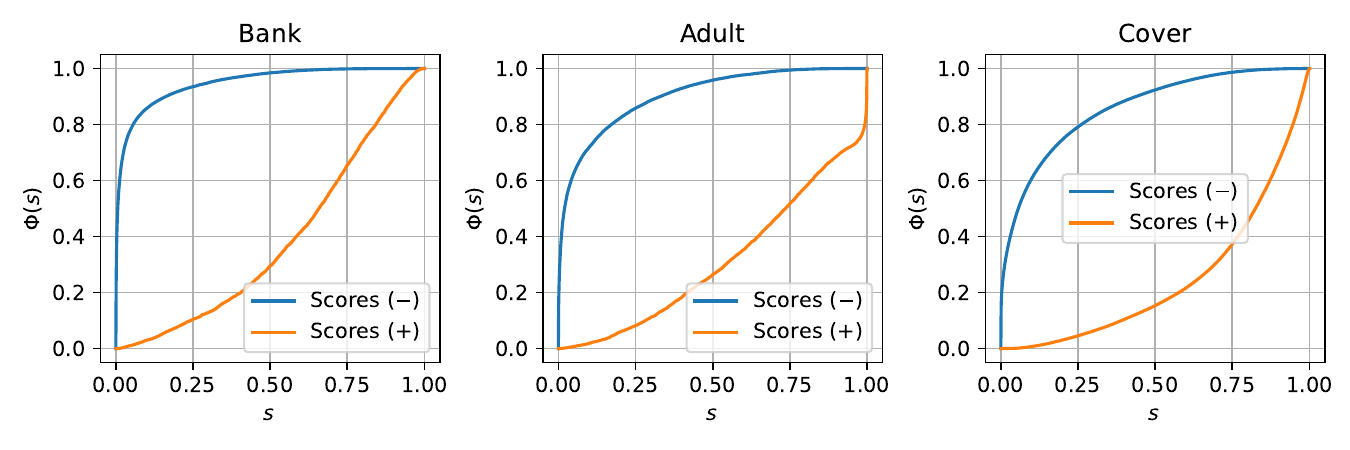}%
\caption{ECDF of prediction scores (XGBoost).}%
\label{fig:ecdf-xgb}%
\Description[]{}%
\end{figure*}

\smallskip
\noindent
\textbf{Remark.}
The well-behaved condition %(Definition~\ref{def:well-behaved}) 
is used in our subsequent theoretical analysis to ensure that quantile estimation errors translate to bounded errors in TPR and FPR. 
We note that this condition may not hold universally in practice, particularly for score distributions with spikes. Despite this, our empirical results demonstrate robustness across diverse datasets.
For instance, in Figure~\ref{fig:ecdf-xgb}, the Adult dataset exhibits clustering at extremes: for negative examples, 40\% of scores fall below 0.01 and 80\% below 0.17; for positive examples, 20\% exceed 0.993 and 40\% exceed 0.82. 
This clustering violates the Lipschitz condition locally, yet our method still achieves strong empirical accuracy (Section~\ref{sec:evaluation}). 
This robustness may arise because the clustering is only confined to narrow score intervals and has a limited effect on error of the overall curve.
Thus, while the well-behaved assumption is used to prove our theoretical guarantees, the method performs reliably in practice even when this condition is not met.
Meanwhile, the privacy guarantees need no such assumption, and hold regardless.
\begin{lemma}
\label{lem:abs-error-score}
Using the histogram approach for quantile estimation in the federated setting, let the additive error of quantile estimation be bounded by $\alpha$, assuming the score distributions are $O(1)$-well-behaved, then we have the following bounds:
\begin{equation}
|T(s)-\hat{T}(s)|\le O(\alpha)
\quad\mathrm{and}\quad
|F(s)-\hat{F}(s)|\le O(\alpha)
\end{equation}
\begin{proof}
Since the histogram approach guarantees that the quantile estimation error is bounded:
\begin{equation}
|T^{-1}(t)-\hat{T}^{-1}(t)|\le\alpha
\quad\mathrm{and}\quad
|F^{-1}(f)-\hat{F}^{-1}(f)|\le\alpha
\label{eq:quantile-error}
\end{equation}
Since $T(s) = 1 - \Phi^+(s)$ and $F(s) = 1 - \Phi^-(s)$, by $O(1)$-well-behaved distributions, we have:
\begin{equation}
|T(s_1)-T(s_2)|\le O(|s_1-s_2|) \; \mathrm{and}\;
|F(s_1)-F(s_2)|\le O(|s_1-s_2|)
\end{equation}
Then the absolute error can be bounded as follows:
\begin{align}
|T(s)-\hat{T}(s)|
&=|T(s)-T(T^{-1}(\hat{T}(s)))| \\
&\le O(|s-T^{-1}(\hat{T}(s))|) \\
&=O(|\hat{T}^{-1}(\hat{T}(s))-T^{-1}(\hat{T}(s))|)
\le O(\alpha)
\end{align}
Similarly, we have $|F(s)-\hat{F}(s)|\le O(\alpha)$.
\end{proof}
\end{lemma}
Next, we can use Lemma~\ref{lem:abs-error-score} to bound the error of $T(F^{-1}(f))$ and $F(T^{-1}(t))$, which we write as $T(f)$ and $F(t)$ for shorthand.
\begin{lemma}
\label{lem:abs-error-curve}
Using the histogram approach for quantile estimation in the federated setting, let the additive error of quantile estimation be bounded by $\alpha$, and assuming the score distributions are $O(1)$-well-behaved, then we have the following bounds:
\begin{equation}
|T(f)-\hat{T}(f)|\le O(\alpha)
\quad\mathrm{and}\quad
|F(t)-\hat{F}(t)|\le O(\alpha)
\end{equation}
\end{lemma}
\begin{proof}
By $O(1)$-well-behaved distributions and additive error of quantile estimation, we have:
\begin{equation}
|T(F^{-1}(f))-T(\hat{F}^{-1}(f))|\le O(|F^{-1}(f)-\hat{F}^{-1}(f)|)\le O(\alpha)
\label{eq:abs-error-tpr-mid}
\end{equation}
Now consider the absolute error:
\begin{align}
&|T(f)-\hat{T}(f)|\notag\\
=&|T(F^{-1}(f))-\hat{T}(\hat{F}^{-1}(f))|\\
=&|T(F^{-1}(f))-T(\hat{F}^{-1}(f))+T(\hat{F}^{-1}(f))-\hat{T}(\hat{F}^{-1}(f))|\\
\le &|T(F^{-1}(f))-T(\hat{F}^{-1}(f))|+|T(\hat{F}^{-1}(f))-\hat{T}(\hat{F}^{-1}(f))|\\
\le & O(\alpha)+O(\alpha)=O(\alpha) \label{eq:abs-error-tpr-f}
\end{align}
Similarly, we have $|F(t)-\hat{F}(t)|\le O(\alpha)$.
\end{proof}
\noindent
\textbf{Remark.}
We note that our analysis (here, and subsequently) only requires that the interpolation of the ECDF is monotone.
Hence, both linear interpolation and PCHIP satisfy this property.
Consequently, they yield the same asymptotic Area Error bounds.
Our recommendation is for PCHIP, as we observe that this leads to smoother curves and improved area error in practice.
\subsection{Bounding the Area Error Using Approximate Quantiles}
\label{sec:ae-approx-q}
Next, we analyze the Area Error of the ROC and PR curves when using approximate quantiles.
In the federated setting, we use the histogram approach to estimate quantiles, which introduces an additive error $\alpha$ in quantile estimation. This error propagates to errors in TPR and FPR, which in turn affect the Area Error of the ROC and PR curves. Using the bounds established in Lemma~\ref{lem:abs-error-curve}, we can derive bounds on the Area Error for both curves under the assumption of well-behaved score distributions.
\begin{theorem}[ROC-AE-ApproxQ]
\label{thm:ae-roc-approx-q}
If $Q$ \textbf{approximate} quantiles are used for both positive and negative examples, under the assumption of $O(1)$-well-behaved score distributions, the Area Error between the true and estimated ROC curves is bounded by $O(1/Q)$.
\end{theorem}
\begin{proof}[Proof]
Let the quantile estimation error be $\alpha = O(1/Q)$, which is the width of the finest histogram bin.
By well-behaved score distributions and Lemma~\ref{lem:abs-error-curve}, we have $|T(f)-\hat{T}(f)|\le\alpha$.
Therefore, the total area error is the integral of the absolute error:
\begin{equation}
\text{AE}_\text{ROC}=\int_0^1 |T(f)-\hat{T}(f)|df\le\alpha=O(1/Q)
\end{equation}
\end{proof}
The analysis for the PR curve is similar in outline, but requires additional care. 
Precision depends on both $T(s)$ and $F(s)$ via Equation~\eqref{eq:precision}, and is also affected by class imbalance \citep{Williams2021}. 
In the worst case (when $t$ is close to 0), the absolute error $|P(t) - \hat{P}(t)|$ can be close to 1. When $t$ is close to 1, the absolute error is as high as $O(1/Q)$. However, we can bound the AE by integrating the absolute error over the entire range of TPR (Recall).
\begin{theorem}[PR-AE-ApproxQ]
\label{thm:ae-pr-approx-q}
If $Q$ \textbf{approximate} quantiles are used for both positive and negative examples, under the assumption of $O(1)$-well-behaved score distributions and mild class imbalance ($r=n^+/n^-\ge0.1$), the Area Error between the true and estimated PR curves is bounded by $\tilde{O}(1/Q)$. Under extreme class imbalance ($r\ll1$), the AE is bounded by $\tilde{O}(1/(Qr))$.
\end{theorem}
\begin{proof}[Proof]
AE is the integral of $|P(t) - \hat{P}(t)|$ over $t \in [0,1]$. By bounding $|P - \hat{P}|$ as a function of $t$, and integrating, we obtain a worst-case Area Error. However, under extreme class imbalance, the bound is also dependent on the ratio $r$. See full details below.

Since precision is dependent on TPR, we can bound the absolute error in precision $|P(t)-\hat{P}(t)|$ as a function of $t$, then compute the area error by integrating over the TPR range. Given Equation~\eqref{eq:precision}, we reformulate the precision $P$ as a function of $t$:
\begin{equation}
P(t)=\frac{t\cdot n^+}{t\cdot n^++F(t)\cdot n^-}
\end{equation}
We first analyze the perfectly balanced setting (Case 1)
and then extend to imbalanced settings (Case 2). 
\subsubsection*{Case 1: Perfectly Balanced Class Distribution (\texorpdfstring{$r=1$}{r=1})}
\label{sec:proof-ae-pr-balanced}
To build an intuition, we first consider the case where the number of positive examples is equal to the number of negative examples.
This gives the pattern of the proof that we subsequently extend to arbitrary ratios of positive to negative examples. Since $n^+ = n^-$, we have:
\begin{align}
P(t)=\frac{t}{t+F(t)}
\quad\mathrm{and}\quad
\hat{P}(t)=\frac{t}{t+\hat{F}(t)}
\end{align}
Let the absolute error in FPR at given $t$ be bounded, i.e., $|F(t)-\hat{F}(t)|\le\alpha$, where $\alpha=O(1/Q)$ (Lemma~\ref{lem:abs-error-curve}).
To maximize the error $P(t)-\hat{P}(t)$, we minimize $\hat{P}(t)$ by maximizing $\hat{F}(t)$. We set $\hat{F}(t) = F(t) + \alpha$. This gives:
\begin{align}
P(t)-\hat{P}(t)
&\le\frac{t}{t+F(t)}-\frac{t}{t+F(t)+\alpha}\\
&=\frac{t\cdot\alpha}{(t+F(t))(t+F(t)+\alpha)}
\label{eq:plb}
\end{align}
For the other direction, to maximize $\hat{P}(t)-P(t)$, we maximize $\hat{P}(t)$ by minimizing $\hat{F}(t)$. We set $\hat{F}(t)=F(t)-\alpha$. This gives:
\begin{align}
\hat{P}(t)-P(t)
&\le\frac{t}{t+F(t)-\alpha}-\frac{t}{t+F(t)}\\
&=\frac{t\cdot\alpha}{(t+F(t))(t+F(t)-\alpha)}
\label{eq:pub}
\end{align}
Therefore, the absolute error in precision is bounded by the larger of \eqref{eq:plb} and \eqref{eq:pub}, which is \eqref{eq:pub}:
\begin{align}
|\hat{P}(t)-P(t)|
&\le\frac{t\cdot\alpha}{(t+F(t))(t+F(t)-\alpha)}\\
&=\frac{\alpha}{(1+F(t)/t)(t+F(t)-\alpha)}\\
&\le\frac{\alpha}{t+F(t)-\alpha}
\le\frac{\alpha}{t-\alpha}
\end{align}
Note that the precision is in the range of $[0,1]$, and the maximum absolute error is 1. Therefore, for $t\le 2\alpha$, the absolute error is bounded by 1; for $t>2\alpha$, the absolute error is bounded by $\alpha/(t-\alpha)$, which is smaller than 1. Integrating this bound over $t$ yields the total area error:
\begin{align}
\text{AE}_\text{PR}
&=\int_0^1 |P(t)-\hat{P}(t)|dt\\
&\le 2\alpha+\int_{2\alpha}^1 \frac{\alpha}{t-\alpha}dt\\
&=2\alpha+\alpha\log(1/\alpha-1)\\
&\le2\alpha+\alpha\log(1/\alpha)\\
&=O(\log Q/Q)=\tilde{O}(1/Q)
\end{align}
For the last term, since $\alpha\log(1/\alpha)=\log Q/Q$ and $\log Q$ grows much more slowly than $Q$, we may treat it as a constant, and think of this term as being dominated by $1/Q$. 
Pragmatically, setting $Q = 10^4$ gives AE $\approx10^{-3}$, which is sufficient for most applications.
\subsubsection*{Case 2: Imbalanced Class Distribution (\texorpdfstring{$r \ne 1$}{r ≠ 1})}
\label{sec:proof-ae-pr-imbalanced}
Under the imbalanced class distribution, we consider the case to maximize $\hat{P}(t)-P(t)$ by maximizing $\hat{P}(t)$. For the other direction, the results are similar so we omit for brevity.
The absolute error in precision is bounded by:
\begin{align}
\hat{P}(t)-P(t)
&=\frac{t\cdot n^+}{t\cdot n^++\hat{F}(t)\cdot n^-}-\frac{t\cdot n^+}{t\cdot n^++F(t)\cdot n^-}\\
&\le\frac{t\cdot n^+}{t\cdot n^++(F(t)-\alpha)\cdot n^-}-\frac{t\cdot n^+}{t\cdot n^++F(t)\cdot n^-}\\
&=\frac{n^+n^-\cdot t}{t\cdot n^++F(t)\cdot n^-}\cdot\frac{\alpha}{t\cdot n^++F(t)\cdot n^--\alpha\cdot n^-}\\
&=\frac{n^+n^-}{n^++(F(t)/t)\cdot n^-}\cdot\frac{\alpha}{t\cdot n^++F(t)\cdot n^--\alpha\cdot n^-}
\end{align}
Let $r = n^+/n^-$. We distinguish two sub-cases:

\smallskip
\noindent
\textbf{Subcase (a) ($r>1$)}:
The absolute error is bounded by:
\begin{align}
|\hat{P}(t)-P(t)|
&\le\frac{n^-}{n^++(F(t)/t)\cdot n^-}\cdot\frac{n^+\cdot\alpha}{t\cdot n^++F(t)\cdot n^--\alpha\cdot n^-}\\
&=\frac{1}{r+F(t)/t}\cdot\frac{\alpha}{t+F(t)/r-\alpha/r}\\
&\le\frac{1}{r}\cdot\frac{\alpha}{t-\alpha/r}
\le\frac{\alpha}{t-\alpha/r}
\le\frac{\alpha}{t-\alpha}
\end{align}
This is the same as the perfectly balanced class distribution, the area error is bounded by $\tilde{O}(1/Q)$.

\smallskip
\noindent
\textbf{Subcase (b) ($r<1$)}: The absolute error is bounded by:
\begin{align}
|\hat{P}(t)-P(t)|
&\le\frac{n^+}{n^++(F(t)/t)\cdot n^-}\cdot\frac{n^-\cdot\alpha}{t\cdot n^++F(t)\cdot n^--\alpha\cdot n^-}\\
&=\frac{1}{1+F(t)/(t\cdot r)}\cdot\frac{\alpha}{t\cdot r+F(t)-\alpha}
\le\frac{\alpha}{t\cdot r-\alpha}
\end{align}
The Area Error is bounded by:
\begin{align}
\text{AE}_\text{PR}
&=\int_0^1 |P(t)-\hat{P}(t)|dt\\
&\le 2\alpha/r+\int_{2\alpha/r}^1 \frac{\alpha}{t\cdot r-\alpha}dt\\
&=2\alpha/r+\alpha/r\cdot\log(r/\alpha-1)\\
&\le2\alpha/r+\alpha/r\cdot\log(r/\alpha)
\end{align}
Therefore, provided that $Q$ is sufficiently large relative to $1/r$ (so that $2\alpha/r<1$), the Area Error is bounded by $\tilde{O}(\alpha/r)=\tilde{O}(1/(Qr))$. For mild class imbalance with $r\ge0.1$, this reduces to $\tilde{O}(1/Q)$.
\end{proof}
Next, we extend the analysis to the federated setting with security and privacy constraints, which could introduce additional sources of error from privacy noise.
\subsection{Error Analysis in the Federated Setting with Security and Privacy}
\label{sec:proof-ae-federated}
We analyze the Area Error separately for the Secure Aggregation (SA) and Distributed Differential Privacy (DDP) settings. In SA, the main source of error is from the approximate quantile estimation due to histogram binning. In DDP, there is additional error from the privacy noise added to the histogram counts. We combine these errors to derive bounds on the Area Error for both the ROC and PR curves under these settings.
\begin{theorem}[AE-SA]
\label{thm:ae-sa}
Under Secure Aggregation, assuming $O(1)$-well-behaved score distributions, the Area Error between the true and estimated curves is bounded by $O(1/Q)$ for the ROC curve and $\tilde{O}(1/Q)$ for the PR curve.
\end{theorem}
\begin{proof}
Under SA without privacy noise, the quantile estimation error is bounded by the histogram bin width $\alpha = O(1/Q)$. Applying Theorems~\ref{thm:ae-roc-approx-q} and~\ref{thm:ae-pr-approx-q} directly yields the bounds.
\end{proof}
\begin{theorem}[AE-DDP]
\label{thm:ae-ddp}
Under Distributed Differential Privacy with budget $\varepsilon$, assuming $O(1)$-well-behaved score distributions, the Area Error between the true and estimated curves is bounded by $\tilde{O}(1/Q + 1/(n^+\varepsilon) + 1/(n^-\varepsilon))$ for both the ROC and PR curves, where $n^+$ and $n^-$ are the number of positive and negative examples, respectively.
\end{theorem}
\begin{proof}
Under DDP, the privacy budget $\varepsilon$ is split equally across $h$ layers of the hierarchical histogram. Each bin incurs Laplace noise with scale $h/\varepsilon$, yielding variance $O(h^2/\varepsilon^2)$ per bin.

For quantile estimation, we traverse down the histogram tree, selecting the child bin that contains the quantile at each layer. This traversal requires at most $b$ bins per layer and $h$ layers, accumulating total noise variance $O(bh^3/\varepsilon^2)$. The absolute error of the quantile estimate is thus $O(\sqrt{bh^3}/(n^*\varepsilon))$ when normalized by the number of examples $n^*$, where $n^*$ is equal to $n^+$ or $n^-$ for positive or negative examples, respectively.
For typical settings with $b = 2$, $h = O(\log Q)$, and $n^* \gg \sqrt{bh^3}$, the expected quantile estimation error is bounded by
$\alpha = \tilde{O}(1/Q + 1/(n^*\varepsilon))$.
For simplicity, we denote $\alpha^+=\tilde{O}(1/Q + 1/(n^+\varepsilon))$ and $\alpha^-=\tilde{O}(1/Q + 1/(n^-\varepsilon))$ for the quantile estimation errors of positive and negative examples, respectively.

Plugging the quantile error $\alpha^+$ and $\alpha^-$ into Eq.~\eqref{eq:quantile-error} in Lemma~\ref{lem:abs-error-score}, we have $|T(s)-\hat{T}(s)|\le O(\alpha^+)$ and $|F(s)-\hat{F}(s)|\le O(\alpha^-)$.
Meanwhile, Eq.~\eqref{eq:abs-error-tpr-mid} in Lemma~\ref{lem:abs-error-curve} will be bounded by $O(\alpha^-)$.
Therefore, Eq.~\eqref{eq:abs-error-tpr-f} in Lemma~\ref{lem:abs-error-curve} will be bounded by $O(\alpha^+ + \alpha^-)$, i.e., $|T(f)-\hat{T}(f)|\le O(\alpha^+ + \alpha^-)$. Similarly, we have $|F(t)-\hat{F}(t)|\le O(\alpha^+ + \alpha^-)$.
Since Theorem~\ref{thm:ae-roc-approx-q} only uses the bound of $|T(f)-\hat{T}(f)|$ and Theorem~\ref{thm:ae-pr-approx-q} only uses the bound of $|F(t)-\hat{F}(t)|$, we can replace the quantile error $\alpha$ with $O(\alpha^+ + \alpha^-)$, which yields the Area Error bounds for the ROC and PR curves under DDP.

The class counts $n^+$ and $n^-$ are estimated from the noisy histograms with absolute error $O(h/\varepsilon) = O(\log Q/\varepsilon)$. For typical datasets and privacy budgets (e.g., $\varepsilon \approx 1$), this error is negligible relative to $n^+$ and $n^-$. For the ROC curve, these counts are not needed. For the PR curve, the error is dominated by the quantile estimation error term and does not change the overall bound.
\end{proof}
\section{Empirical Evaluation}
\label{sec:evaluation}
We evaluate our method on three datasets: Bank~\citep{Moro2012}, Adult~\citep{Becker1996}, and Cover~\citep{Blackard1998}. To account for the sensitivity of class imbalance, we select datasets with varying positive-to-negative ratios $r$. Dataset statistics are summarized in Table~\ref{tab:data-info}. 
We evaluate score functions produced by two classifiers: XGBoost and Logistic Regression (see the Appendices~\ref{sec:expt-pr-strategy}, and~\ref{expt-pr-xgb}, \ref{expt-logistic-regression} for more results); recall that our results apply to any score function, so the nature of the underlying classifier is not a major factor. 
We set the branching factor $b=2$ and the height $h = \lceil \log_2 Q \rceil + 2$. For differential privacy, we use distributed discrete Laplace noise and apply post-processing by default.
\begin{table}[t]
\caption{Dataset statistics.}
\label{tab:data-info}
\centering%
\begin{tabular}{cccc}
\toprule
Datasets & \# Row  & \# Col & Pos-to-Neg Ratio \\
\midrule
Bank   & 45K  & 16 & 0.132 \\
Adult  & 33K  & 14 & 0.317 \\
Cover* & 581K & 54 & 0.574 \\
\bottomrule
\multicolumn{4}{l}{\small *Class 1 is treated as positive; others as negative.}
\end{tabular}
\end{table}
\subsection{Initial Experiments}
\label{sec:initial-expt}
\textbf{Performance Evaluation.}
We integrated our method into the Flower federated learning framework~\citep{Beutel2020} and evaluated its performance using the SecAgg+ protocol~\citep{Bonawitz2017} for secure histogram aggregation. We measured the client-side execution time and memory consumption while varying the number of histogram bins. The results are shown in Figure~\ref{fig:cost-time-mem}. Since SecAgg+ consists of four phases, we exclude the setup and key sharing phases and report only the cost of the histogram aggregation phase, which is the component introduced by our method. Consistent with the communication analysis in Section~\ref{sec:quantile-histogram}, our method incurs very little communication overhead (less than 1 MB per client), while the overall execution time and memory usage are dominated by the underlying secure aggregation protocol. Across all configurations, the client-side aggregation takes approximately 1.5 seconds and consumes around 450 MB of memory (mostly from the secure aggregation protocol and Flower itself), with both metrics remaining nearly constant as the number of histogram bins increases.
\begin{figure}[t]
\centering%
\includegraphics[width=\columnwidth]{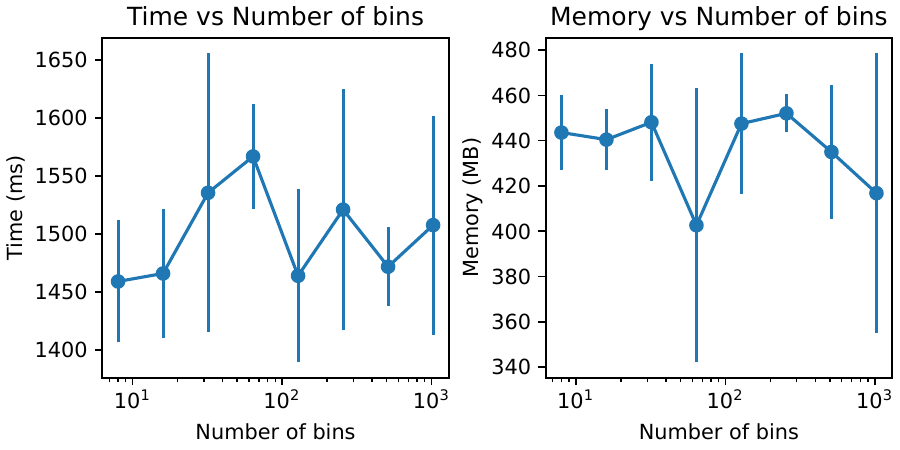}%
\caption{Time and memory performance.}%
\label{fig:cost-time-mem}%
\Description[]{}%
\end{figure}

\smallskip
\noindent
\textbf{Interpolation Methods.}
\label{evaluation-interp}
We compare linear interpolation and piecewise cubic Hermite interpolation (PCHIP) for the interpolation step (Section~\ref{sec:curve-approx}). 
Results are shown in Figure~\ref{fig:interp-roc-xgb},
%(more plots shown in Figure~\ref{fig:interp-pr-xgb}, \ref{fig:interp-roc-lr}, and \ref{fig:interp-pr-lr} in the Appendix)
under Secure Aggregation (SA, no privacy noise), and Distributed Differential Privacy (DDP, $\varepsilon=1$ or $0.3$). 
The PCHIP method consistently outperforms linear interpolation, though the margin is small in some cases. Since PCHIP yields smoother curves that more closely approximate the true ROC and PR curves, we recommend using it in practice.
In Figure~\ref{fig:interp-roc-xgb}, we see that under SA the AE decreases as $Q$ increases, whereas under DDP there is a plateau due to privacy noise, as predicted by our analysis (Theorem~\ref{thm:ae-ddp}).
\begin{figure*}[t]
\centering%
\includegraphics[width=\figwidth]{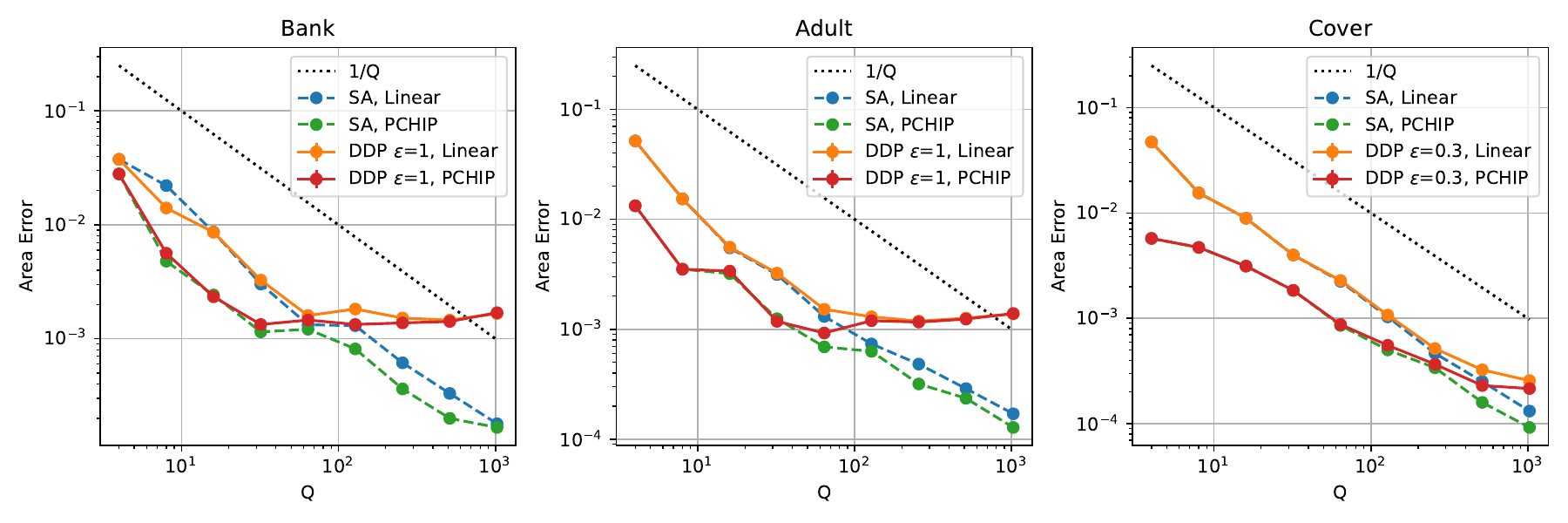}%
\caption{Interpolation method comparison (ROC, XGBoost).}%
\label{fig:interp-roc-xgb}%
\Description[]{}%
\end{figure*}

\smallskip
\noindent
\textbf{Privacy-Accuracy Trade-off.}
\label{evaluation-epsilon}
We evaluate performance under varying privacy levels $\varepsilon\in\{0.1,0.3,1\}$. As shown in Figure~\ref{fig:epsilon-roc-xgb},
%(more plots shown in Figure~\ref{fig:epsilon-pr-xgb}, \ref{fig:epsilon-roc-lr}, and \ref{fig:epsilon-pr-lr} in the Appendix)
smaller $\varepsilon$ results in higher area error. Nonetheless, the error remains low even under strong privacy guarantees ($\varepsilon \le 1$). Larger datasets such as Cover exhibit smaller error due to more accurate quantile estimation under DP (recall, our error bound scales with $1/(n\varepsilon)$). Moreover, the error of the PR curve is slightly higher than that of the ROC curve, since the error includes a logarithmic factor in $Q$ for the PR curve, while it is linear in $1/Q$ for the ROC curve (Theorems~\ref{thm:ae-roc-approx-q} and \ref{thm:ae-pr-approx-q}).

We also compare the Exact Quantiles (EQ), as discussed in Section~\ref{sec:ae-bounds-exact-q} and Secure Aggregation of histograms (SA), as discussed in Section~\ref{sec:proof-ae-federated}. 
In Figure~\ref{fig:epsilon-roc-xgb}, for Cover dataset the gap between the two is more noticeable, while it is negligible on datasets like Adult and Bank, especially when $Q$ is large. This is because the quantile estimation error is smaller with finer histogram bins (larger $Q$).
\begin{figure*}[t]
\centering%
\includegraphics[width=\figwidth]{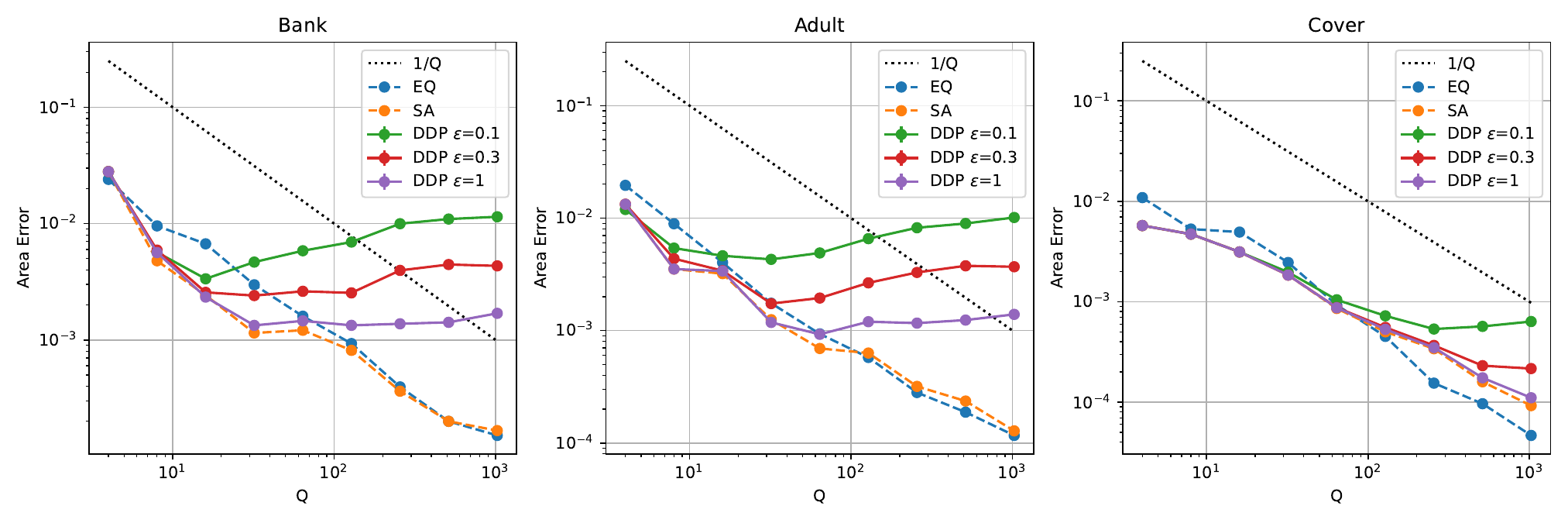}%
\caption{Effect of $\varepsilon$ (ROC, XGBoost).}%
\label{fig:epsilon-roc-xgb}%
\Description[]{}%
\end{figure*}
\subsection{Comparison with Prior Work}
\label{sec:evaluation-range}
We compare our method with the range-query-based method of \citet{Barczewski2025}, which estimates TPR and FPR at $N$ uniformly spaced thresholds using DP-noised range queries. 
In contrast, our method uses $Q$ quantiles to adaptively select thresholds based on the score distribution. Uniform score thresholds may devote many queries to low-density regions while providing insufficient resolution where scores are concentrated.
However, our quantile-based method ensures that consecutive representative points differ by approximately $1/Q$ in the ECDF, which is better for capturing the curve shape.
To enforce monotonicity, their method introduces smoothing variables $v_i$ added to the TPR and FPR series. The values of $v_i$ are optimized via $\ell_1$ or $\ell_2$ minimization.
In contrast, our method uses post-processing to enforce consistency at the hierarchical histogram levels, which ensures that the estimated quantiles are monotonic. This approach is simpler and more efficient than solving an optimization problem, while still yielding accurate and consistent estimates.
Moreover, we leverage smooth monotone interpolation (PCHIP) to construct ECDFs from the estimated quantiles. Our empirical ablation in Section~\ref{sec:initial-expt} shows that PCHIP provides an additional but comparatively smaller improvement over linear interpolation.
Their analysis bounds the \textit{Squared Point Error} of the ECDF estimates at $N$ uniformly spaced query thresholds by $O(\log^3 N/(n\varepsilon)^2)$\footnote{\citet{Barczewski2025} omits the dependence on $n$ in the statement of their error bound.}, which is $O(\log^{3/2} N/(n\varepsilon))$ for the \textit{Absolute Point Error} at those thresholds. But this guarantee does not control the error between consecutive thresholds. In contrast, under the assumption of $O(1)$-well-behaved score distributions, our bound $O(1/Q+\log^{3/2}Q/(n\varepsilon))$ applies at any score threshold (Lemma~\ref{lem:abs-error-score}). When $Q\approx N$, the privacy-dependent terms have the same asymptotics; however, the guarantees are not directly comparable because our bound includes the term $1/Q$ and requires a Lipschitz condition, whereas their bound applies only at the selected thresholds.

Note that the point error fails to capture cumulative errors over the entire curve, thus we evaluate both methods using Area Error. Empirically, as shown in Figure~\ref{fig:range-roc-xgb}, %(more plots shown in Figures~\ref{fig:range-pr-xgb}, \ref{fig:range-roc-lr}, and \ref{fig:range-pr-lr} in the Appendix)
our method produces more stable area errors with increasing $Q$. 
In contrast, the Range method can show erratic AE results due to increased variance in each bin and added smoothing noise. 
The range method can perform better when $Q < 100$ (lower accuracy regime), especially on the Cover dataset, as it uses fixed thresholds for both TPR and FPR (for example, in Figure~\ref{fig:range-roc-xgb}). 
Our approach, which estimates quantiles separately for positives and negatives, may diverge in threshold alignment, affecting accuracy at small $Q$. 
This discrepancy vanishes as $Q$ increases.
\begin{figure*}[t]
\centering%
\includegraphics[width=\figwidth]{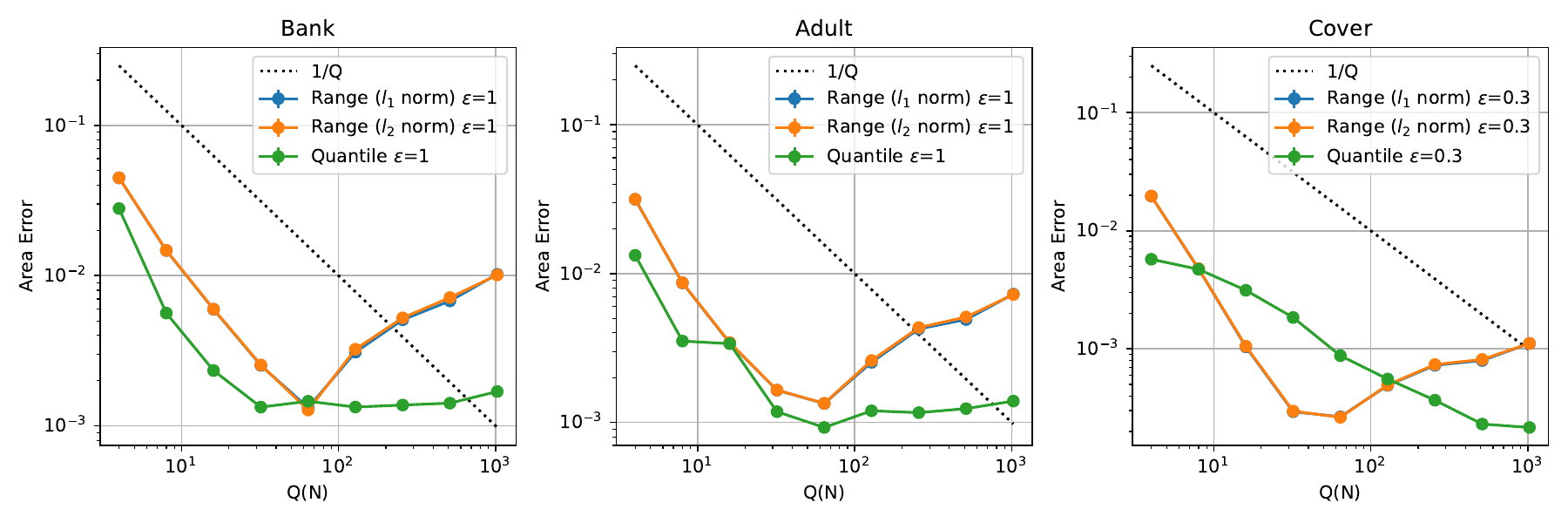}%
\caption{Comparison with range-query method (ROC, XGBoost).}%
\label{fig:range-roc-xgb}%
\Description[]{}%
\end{figure*}
\subsection{Experiments on Class-Imbalanced Data}
\label{sec:evaluation-class-imbalance}
We evaluate the effect of class imbalance on PR curves by varying the positive-to-negative ratio $r$. Specifically, we fix the set of negative examples and randomly subsample the positive examples to achieve the desired ratio. These experiments use the Secure Aggregation method, with results shown in Figure~\ref{fig:ratio-pr-xgb}.
% (and Figure~\ref{fig:ratio-pr-lr} in the Appendix).

In most cases, the AE increases as $r$ decreases, which aligns with the theoretical area error bound of $\tilde{O}(1/(Qr))$. However, there are instances where the observed error is smaller than expected. This occurs because classifier performance degrades significantly on extremely imbalanced datasets, causing the PR curve to approximate that of a random classifier, as a horizontal line at $y = n^+ / n \approx 0$. 
As this is nearly flat, even a poor approximation has low area error.
\begin{figure*}[t]
\centering%
\includegraphics[width=\figwidth]{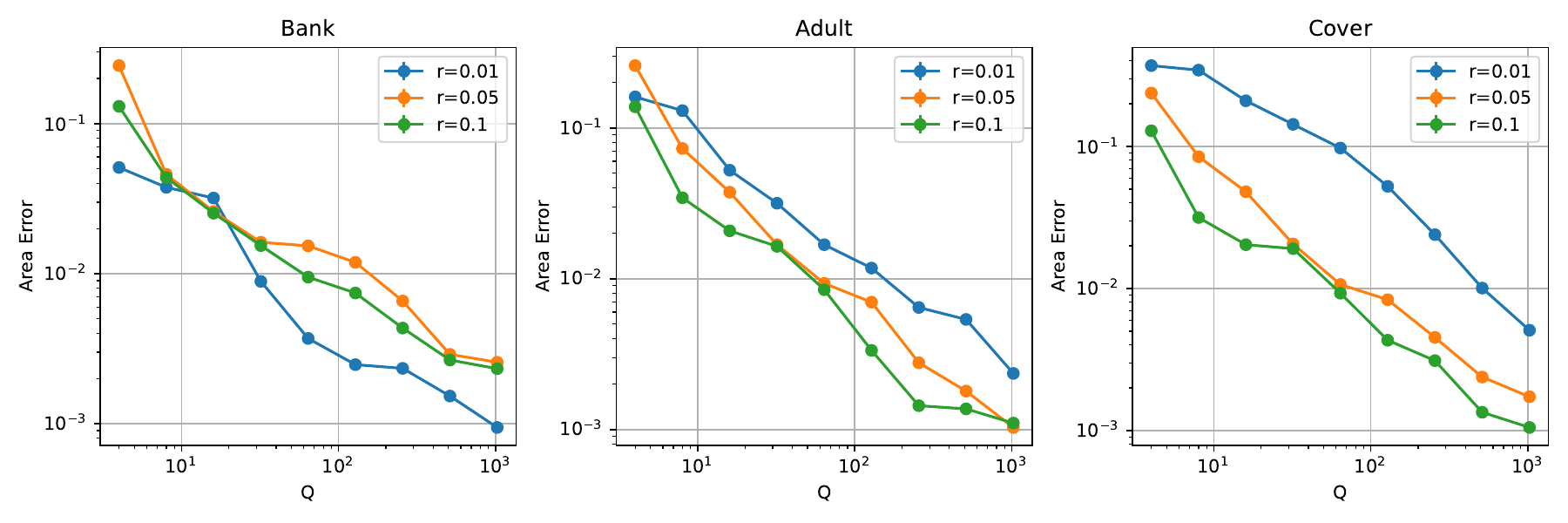}%
\caption{Effect of class imbalance on PR curve (XGBoost).}%
\label{fig:ratio-pr-xgb}%
\Description[]{}%
\end{figure*}
\section{Discussion}
\label{sec:discussion}
\textbf{Extension to Multi-class Setting.}
Our method extends to multi-class classification by decomposing the problem into multiple binary tasks, such as one-vs-rest or one-vs-one schemes. Then, these binary tasks can be evaluated using the proposed quantile-based approach, with AE guarantees holding for each curve. For example, in a one-vs-rest scheme, we can compute a ROC and PR curve for each class against the rest, and then aggregate these curves to obtain multi-class curves such as macro- or weighted-average. The area error bounds carry over accordingly, ensuring that the overall evaluation remains accurate even in the multi-class setting.

\smallskip
\noindent
\textbf{Extension to Other Metrics.}
Our approach extends naturally to derived metrics such as AUC, where the area error directly bounds the absolute AUC difference by $O(1/Q)$, and the Mann-Whitney U-statistic, of which AUC is a special case. 
It also applies to the Detection Error Tradeoff (DET) \citep{Martin1997} curve, which plots the False Negative Rate (FNR) against the FPR on a log scale. Since $\text{FNR} = 1 - \text{TPR}$, we can compute DET similarly to ROC.
Precision-Recall-Gain curves \citep{Flach2015}, which adjust precision and recall by the fraction of positive examples, are another promising direction for extending our approach. We leave their exploration for future work.

\smallskip
\noindent
\textbf{$O(1)$-well-behaved Score Distributions.}
Our analysis shows that this assumption is a sufficient condition for our theoretical analysis, ensuring that quantile estimation errors translate to bounded errors in the ECDFs. 
However, it does not hold universally in practice, particularly for score distributions with spikes as illustrated in Figure~\ref{fig:ecdf-xgb}. Especially in the left panel of Figure~\ref{fig:cover-nb}, the score distribution is highly concentrated around 0 and 1. 
Despite this, our empirical results demonstrate robustness across diverse datasets, even in the extreme case as shown in the left panel of Figure~\ref{fig:cover-nb}. 
Hence the well-behaved assumption is sufficient but not necessary for practical performance;
an open question is whether error guarantees can be proved under a weaker assumption. 
\begin{figure}[t]
\centering%
\includegraphics[width=0.5\columnwidth]{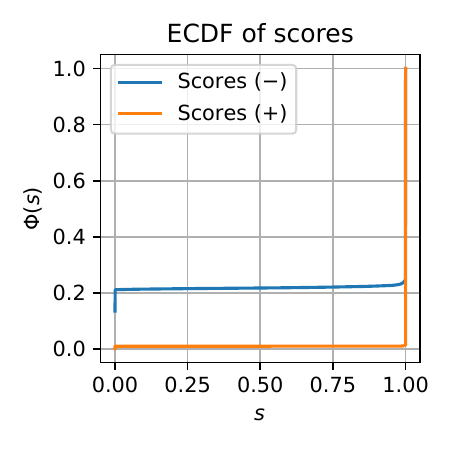}%
\includegraphics[width=0.48\columnwidth]{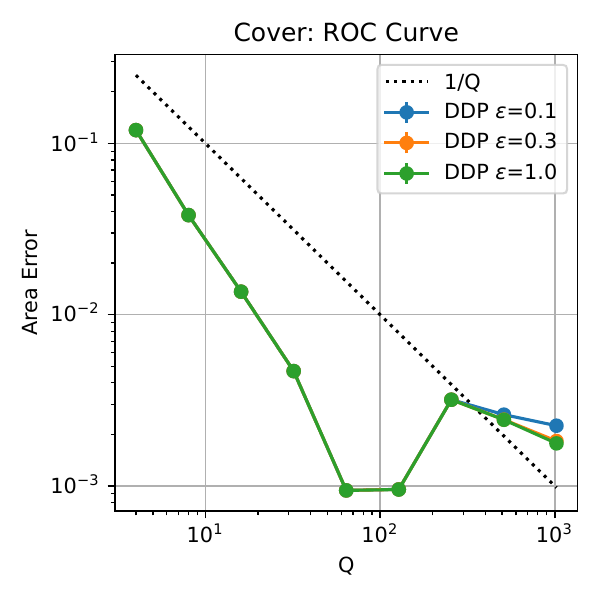}%
\caption{Results on Cover dataset, GaussianNB classifier.}%
\label{fig:cover-nb}%
\Description[]{}%
\end{figure}

\smallskip
\noindent
\textbf{Extension to Local Differential Privacy.}
Our method generalizes to the local DP setting, where each client adds noise independently to their local histogram. In this case, the total noise increases with the number of clients, which is much larger than the distributed DP setting. To simplify the analysis, we assume each client holds one example, which is common in local DP. The total number of clients is equal to the number of examples $n$. The variance of the count in each bin is $O(n/\varepsilon^2)$, and the expected quantile estimation error is $(1/Q+\sqrt{bh}/(\sqrt{n}\varepsilon))$. Therefore, the Area Error is bounded by $O(1/Q + \sqrt{h}/(\sqrt{n}\varepsilon))$; 
see \cite{Cormode2019} for relevant local DP techniques.

\smallskip
\noindent
\textbf{Limitations.}
\label{limitation}
A limitation is that at small $Q$, our method may not outperform range-query-based methods, especially when aligned thresholds are critical. This can be mitigated by hybrid methods combining quantiles and range queries. For the ROC curve, quantiles of all examples can be computed, and TPR/FPR can be derived using range queries, as in \citet{Cormode2023}. Monotonicity can be enforced using post-processing as in \citet{Barczewski2025}. Under mild assumptions, this approach may yield stronger area error bounds for the ROC curve, though it introduces high variance when $Q$ is large and lacks clear area error bounds for the PR curve.
\section{Conclusion}
\label{sec:conclusion}
We studied the problem of computing ROC and PR curves as distributed analytical queries over sensitive data. In federated settings, exact evaluation by centralizing data is disallowed due to privacy constraints and prohibitive communication costs. 
We proposed a quantile-based framework that approximates these curves using compact distribution summaries, enabling efficient and privacy-preserving computation without access to raw data. Our method provides provable guarantees on the Area Error, which scales as $O(1/Q)$ (or $\tilde{O}(1/Q)$), while communication cost grows linearly with the number of quantiles $Q$. This establishes a clear and tunable trade-off between accuracy and communication, a key requirement in large-scale distributed systems. 
The proposed design leverages histogram-based aggregation, making it robust to data heterogeneity and compatible with secure aggregation and distributed differential privacy mechanisms. Experimental results demonstrate that our approach achieves high accuracy under strong privacy constraints, highlighting its practicality for real-world deployments.
For future work, we plan to extend the framework to additional evaluation metrics and monitoring tasks. Another important direction is to strengthen robustness against adversarial or malicious clients, since it is possible for clients to manipulate their local histograms to bias the global curve estimation.
In such cases, additional mechanisms such as by bounding the influence of each client (via succinct non-interactive zero-knowledge proofs) so that each client's contribution falls within certain limits.
\begin{acks}
We thank the anonymous reviewers for their valuable feedback.
XX is supported by a scholarship from the Department of Computer Science, University of Warwick. 
 GC is supported in part by EPSRC grant EP/V044621/1.
\end{acks}
\balance
\bibliographystyle{ACM-Reference-Format}
\bibliography{fedcurve}

\newpage
\appendix
\onecolumn
\input{appendix}
\end{document}

%% file: appendix.tex
\section{Completion of the Proof of Theorem~\ref{thm:ae-roc-exactq}}
\label{sec:posquant}
In the proof of Theorem~\ref{thm:ae-roc-exactq}, we derived an upper bound on the area error (AE) for the ROC curve when using $Q^-$ quantiles for negative examples. However, we omitted the case where the $Q^+$ quantiles for positive examples are also included. In this section, we address this gap and show that the resulting bound remains.
\subsubsection*{(2) Adding \texorpdfstring{$Q^+$}{Q+} Quantiles for Positive Examples}
\label{sec:posquant}
Now we consider incorporating $Q^+$ quantiles for positive examples. If no positive quantile falls between $q^-_i$ and $q^-_{i+1}$, the AE is as previously derived. Otherwise, suppose a positive quantile $q^+_j$ lies in $(q^-_i, q^-_{i+1})$. Then $t_j=T(q^+_j)$ is known exactly by construction, but its associated FPR $\hat{f}_j=\hat{F}(q^+_j)$ may deviate from $f_j=F(q^+_j)$.

Due to ECDF monotonicity, we have $f_i \le f_j \le f_{i+1}$ and $t_i \le t_j \le t_{i+1}$. 
Again, to maximize the worst-case AE, we can assume the true curve lies above the estimated curve:
\begin{equation}
f_i \le f_j \le \hat{f}_j \le f_{i+1}, \quad
\hat{t}_i \le t_i \le t_j \le \hat{t}_{i+1} \le t_{i+1}
\end{equation}
Setting $t_i - \hat{t}_i = t_{i+1} - \hat{t}_{i+1} = \frac{1}{2(Q^+ - 1)}$ and $\hat{f}_j-f_j=\frac{1}{2(Q^- - 1)}$  maximizes AE within the constraints on the quantile values. 
Let $\delta_f=f_j-f_i$ and $\delta_t=t_j-t_i$. This is shown in the left panel of Figure~\ref{fig:neg-pos-q-error}.
\begin{figure}[h]
\centering
\includegraphics[width=0.35\columnwidth]{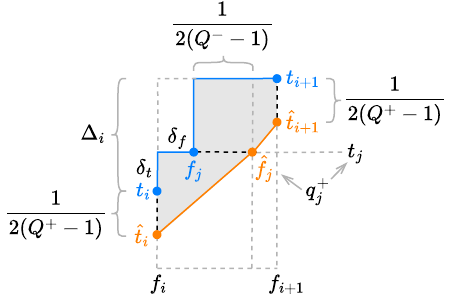}%
\quad%
\includegraphics[width=0.35\columnwidth]{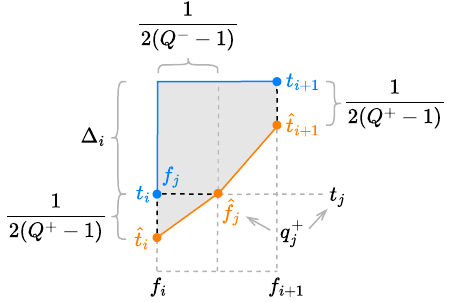}%
\caption{Worst-case AE with an additional positive quantile.}%
\label{fig:neg-pos-q-error}%
\Description[]{}%
\end{figure}

In this configuration, the AE within the strip is composed of two trapezoids minus a rectangle:
\begin{align}
\text{AE}_\text{ROC}^i
\le &
\frac{\frac{1}{2(Q^+-1)}+\Delta_i+\Delta_i-\delta_t}{2}\cdot\Big(\frac{1}{2(Q^--1)}+\delta_f \Big) 
 +\frac{\Delta_i-\delta_t+\frac{1}{2(Q^+-1)}}{2}\cdot\Big(\frac{1}{2(Q^--1)}-\delta_f \Big)
-\delta_f(\Delta_i-\delta_t)
\\
=&\frac{1}{4(Q^+-1)(Q^--1)} + \frac{3\Delta_i}{4(Q^--1)} 
+\underbrace{\delta_t \Big(\delta_f-\frac{1}{2(Q^--1)}\Big) - \frac{\Delta_i\delta_f}{2}}_{\le 0}
\label{eq:approxbound}
\end{align}
The last term contributes a non-positive quantity, since $\delta_f\le \frac{1}{2(Q^--1)}$. Thus, the maximum upper bound in this expression occurs when $\delta_t = \delta_f = 0$, i.e., when $t_j = t_i$ and $f_j = f_i$. In this case, the points $(f_i, t_i)$ and $(f_j, t_j)$ coincide, as shown in the right panel of Figure~\ref{fig:neg-pos-q-error}.

This scenario generalizes: any additional positive quantile within $(q^-_i, q^-_{i+1})$ does not increase AE beyond the bound. So in the worst case
the AE is at the bound from \eqref{eq:approxbound}:
\begin{equation}
\text{AE}_\text{ROC}^i \le
\frac{1}{4(Q^+-1)(Q^--1)} + \frac{3\Delta_i}{4(Q^--1)}
\end{equation}
Adding a positive quantile reduces the first term but increases the second. Summing over all segments, and noting $\sum_i \Delta_i \le 1$, we again obtain an $O(1/Q)$ bound on the total AE.
\section{Additional Discussion}
\noindent
\textbf{Known Support of Prediction Scores.}
Our method assumes that prediction scores lie in a known bounded interval, such as $[0,1]$, as it uses equally spaced bins for histogram construction. But we do not assume prior knowledge of the support of prediction scores, since the scores can be arbitrary values in this interval. However, in some cases, the support may be known. For example, the prediction scores of a $k$-nearest-neighbors classifier often lie on a discrete grid determined by $k$. In such scenarios, the histogram bins can be tailored to the known support to improve estimation accuracy.

\smallskip
\noindent
\textbf{Handling Space Constraints.}
On devices with limited storage (e.g., edge clients), quantile sketches such as GK \citep{Greenwald2001} can reduce space requirements.
Adapting such sketches to our setting would require a privacy-preserving distributed protocol; differentially private quantile summaries such as those studied by \citet{Alabi2023} provide a possible starting point.
\section{Comparison of Different Strategies for PR Curve}
\label{sec:expt-pr-strategy}
We compare two strategies for computing the PR curve: (1) \textbf{Separate:} estimate quantiles for positive and negative examples separately. (2) \textbf{Combine:} estimate quantiles from all examples to approximate the denominator of precision, as shown in Equation~\eqref{eq:precision}.
Results are shown in Figures~\ref{fig:pr-strategy-xgb} and~\ref{fig:pr-strategy-lr}.
The ``separate'' strategy consistently yields lower and more stable area error. This is because it ensures the precision numerator is always less than or equal to the denominator (Equation~\eqref{eq:precision}). In contrast, the ``combine'' strategy may yield a numerator greater than the denominator, requiring clipping to ensure the maximum value is 1.
\begin{figure}[ht]
\centering
\includegraphics[width=\columnwidth]{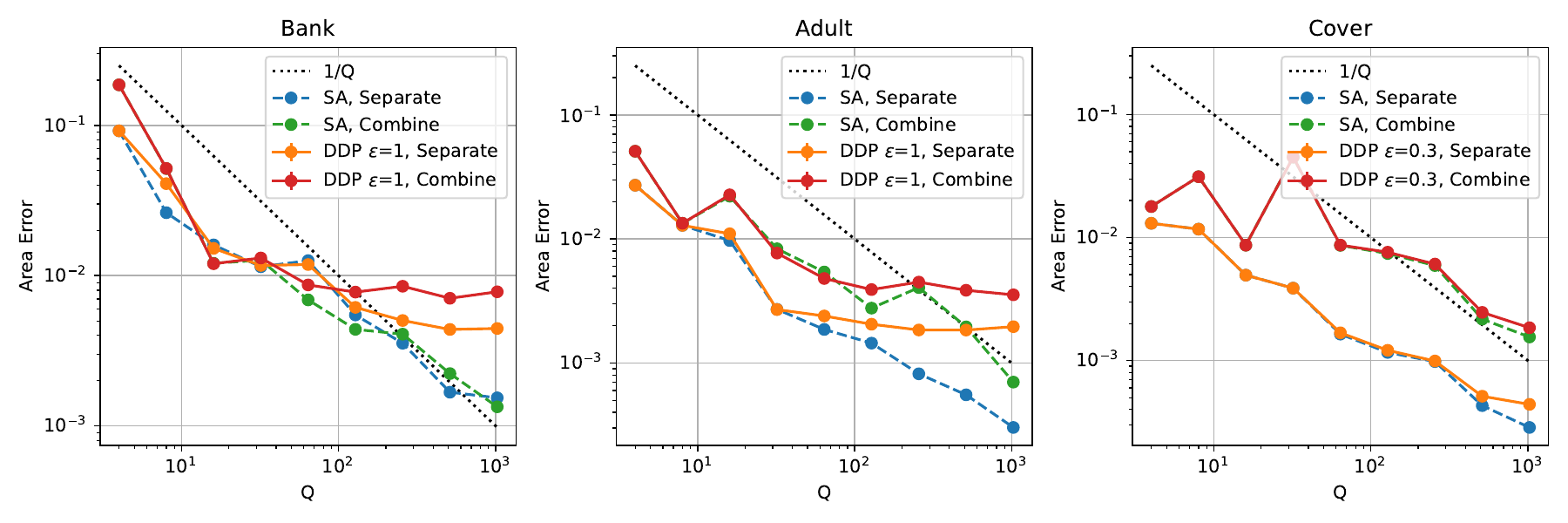}%
\caption{Comparison of strategies for PR curve (XGBoost).}%
\label{fig:pr-strategy-xgb}%
\Description[]{}%
\end{figure}
\begin{figure}[ht]
\centering
\includegraphics[width=\columnwidth]{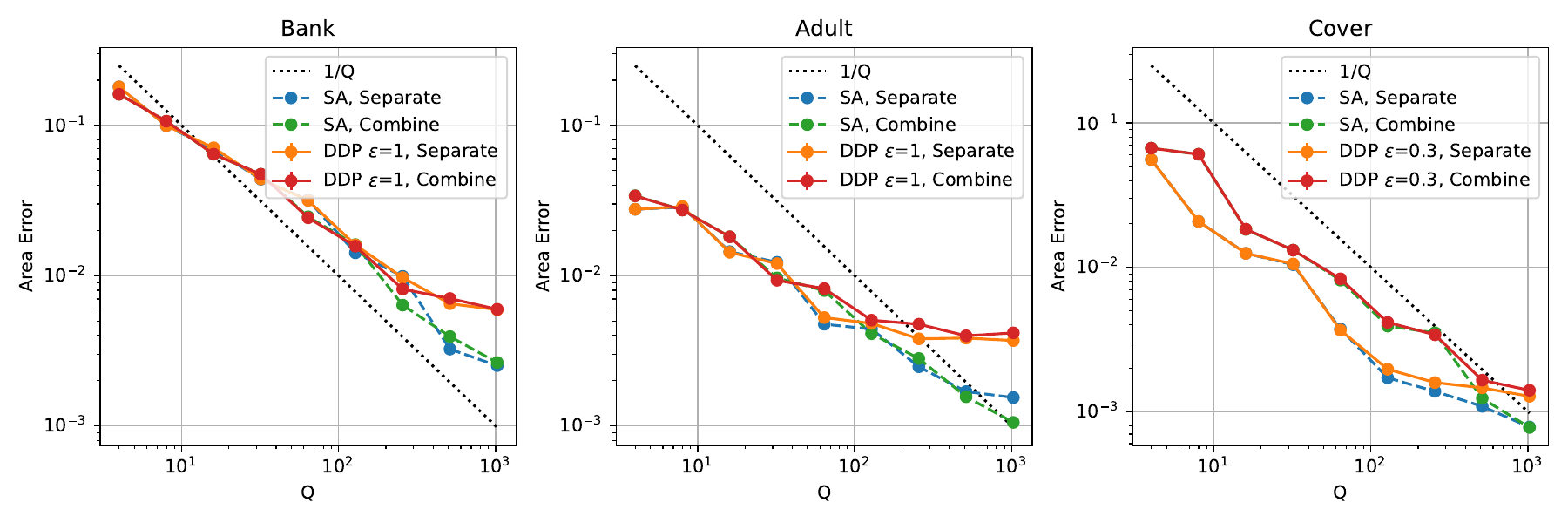}%
\caption{Comparison of strategies for PR curve (Logistic Regression).}%
\label{fig:pr-strategy-lr}%
\Description[]{}%
\end{figure}
\clearpage
\section{PR Curve Experiments with XGBoost}
\label{expt-pr-xgb}
\begin{figure}[ht]
\centering
\includegraphics[width=0.94\columnwidth]{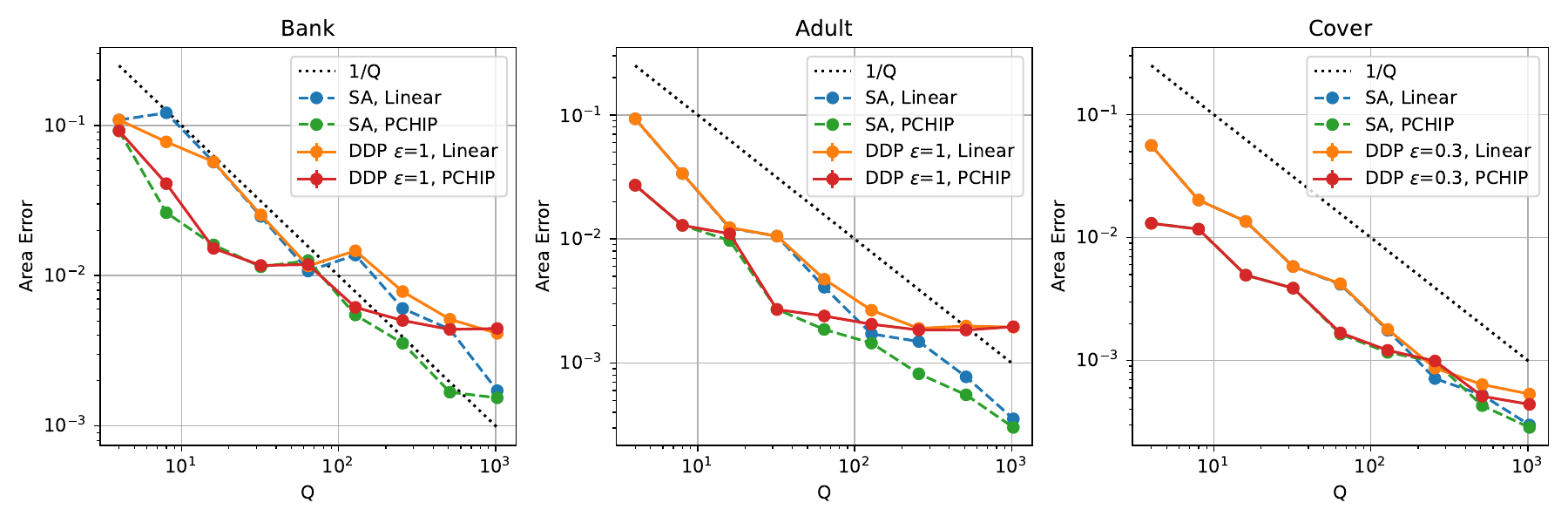}%
\caption{Interpolation method comparison (PR, XGBoost).}%
\label{fig:interp-pr-xgb}%
\Description[]{}%
\end{figure}
\begin{figure}[ht]
\centering
\includegraphics[width=0.94\columnwidth]{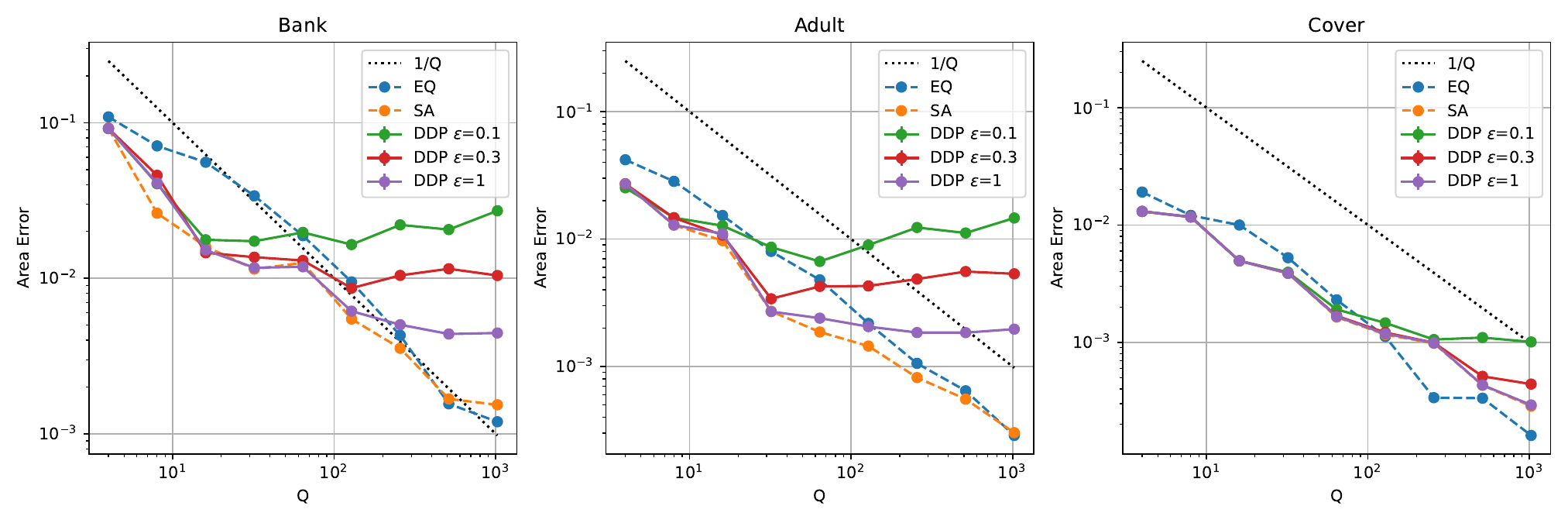}%
\caption{Effect of $\varepsilon$ (PR, XGBoost).}%
\label{fig:epsilon-pr-xgb}%
\Description[]{}%
\end{figure}
\begin{figure}[ht]
\centering
\includegraphics[width=0.94\columnwidth]{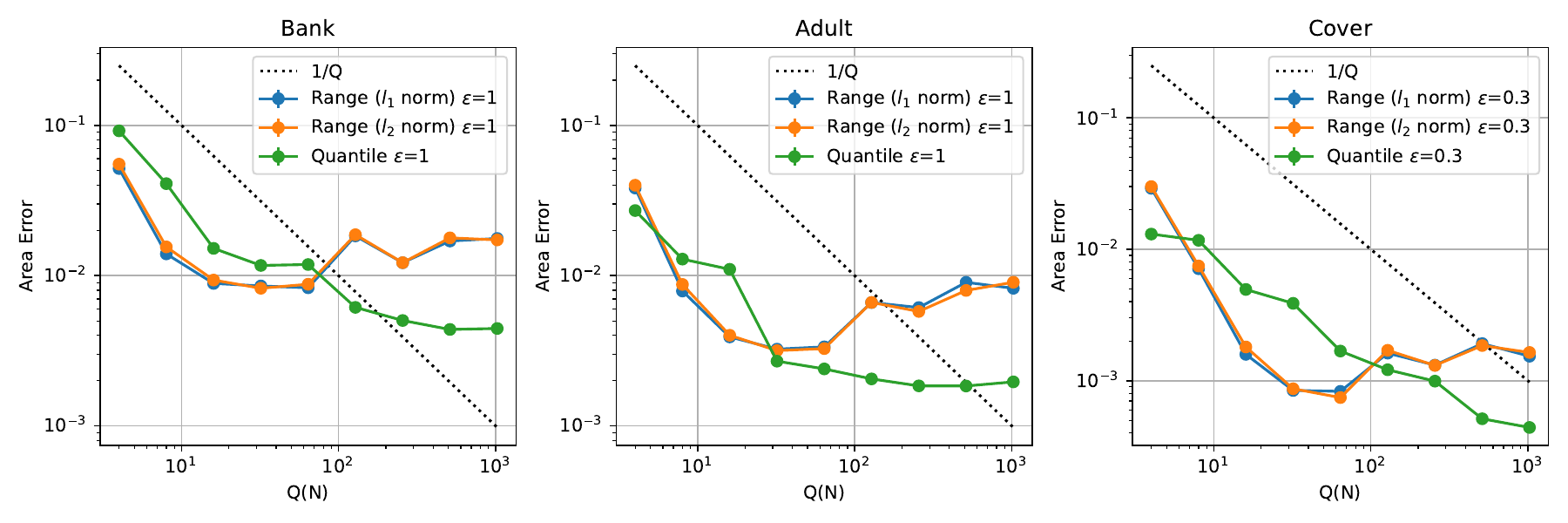}%
\caption{Comparison with range-query method (PR, XGBoost).}%
\label{fig:range-pr-xgb}%
\Description[]{}%
\end{figure}
\clearpage
\section{Experiments with Logistic Regression}
\label{expt-logistic-regression}
\begin{figure}[ht]
\centering
\includegraphics[width=0.94\columnwidth]{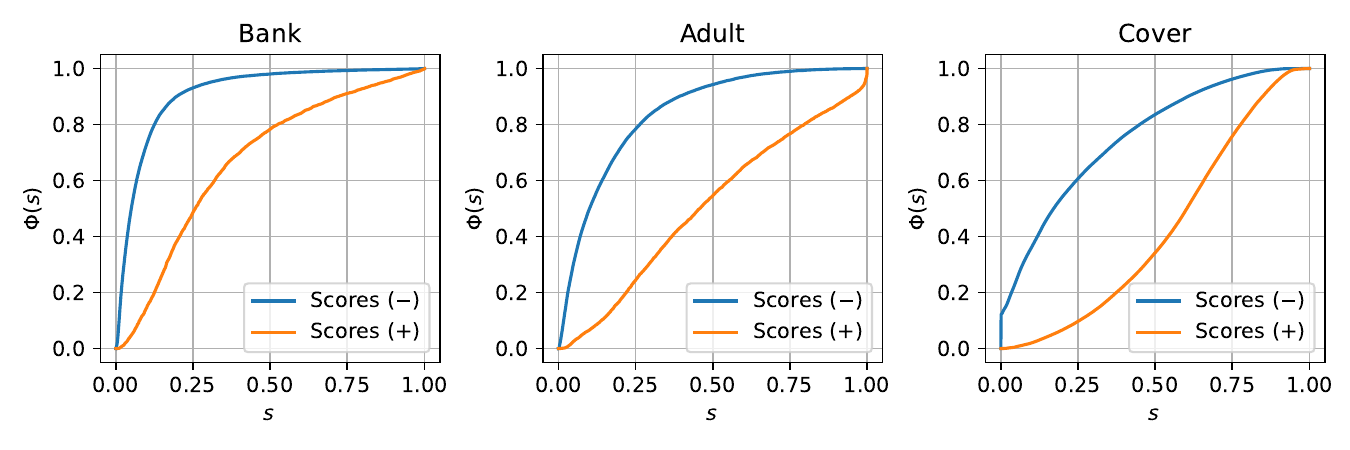}%
\caption{ECDF of prediction scores (Logistic Regression).}%
\label{fig:ecdf-lr}%
\Description[]{}%
\end{figure}
\begin{figure}[ht]
\centering
\includegraphics[width=0.94\columnwidth]{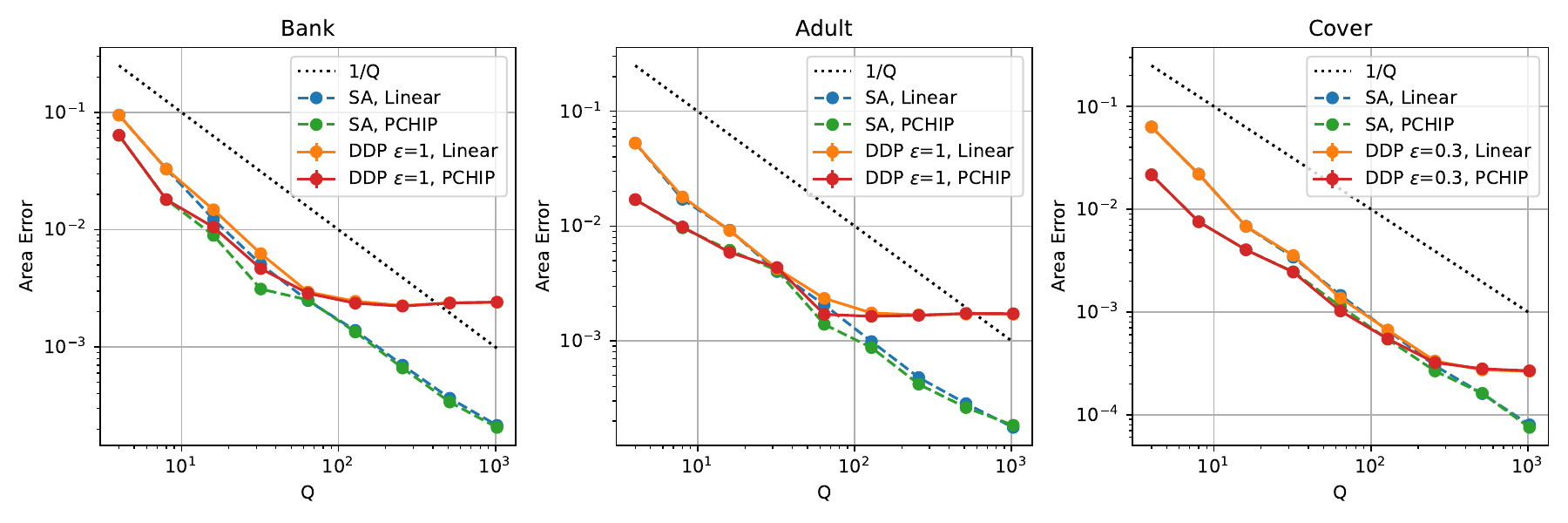}%
\caption{Interpolation method comparison (ROC, Logistic Regression).}%
\label{fig:interp-roc-lr}%
\Description[]{}%
\end{figure}
\begin{figure}[ht]
\centering
\includegraphics[width=0.94\columnwidth]{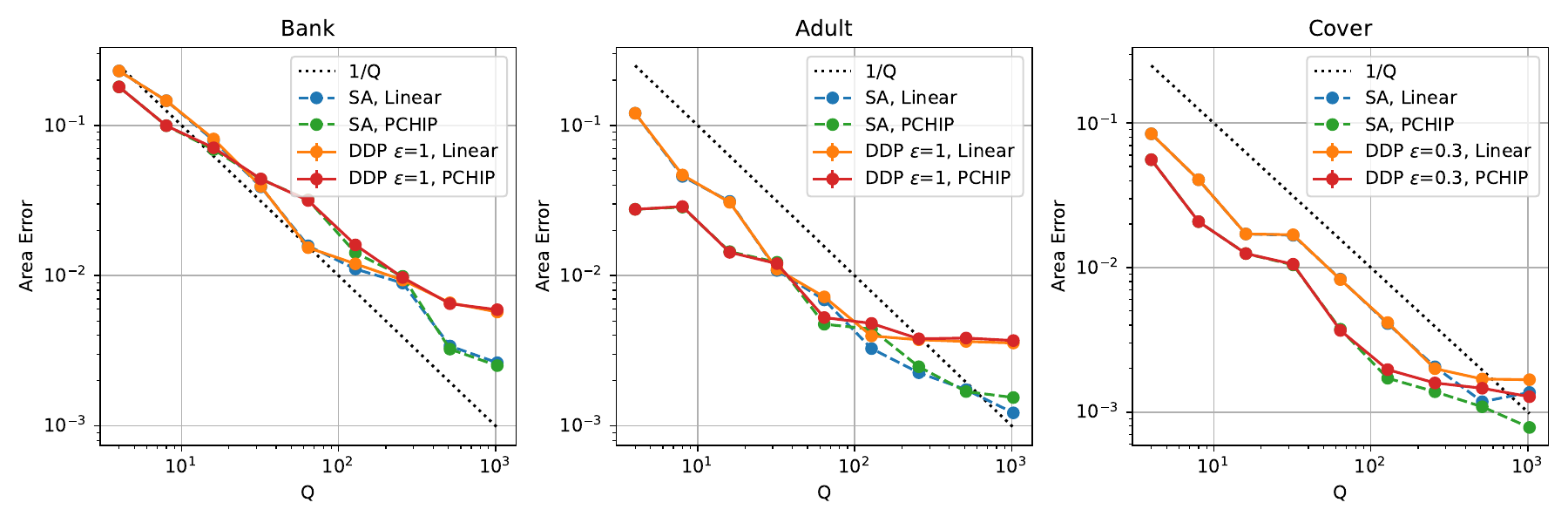}%
\caption{Interpolation method comparison (PR, Logistic Regression).}%
\label{fig:interp-pr-lr}%
\Description[]{}%
\end{figure}
\begin{figure}[ht]
\centering
\includegraphics[width=\columnwidth]{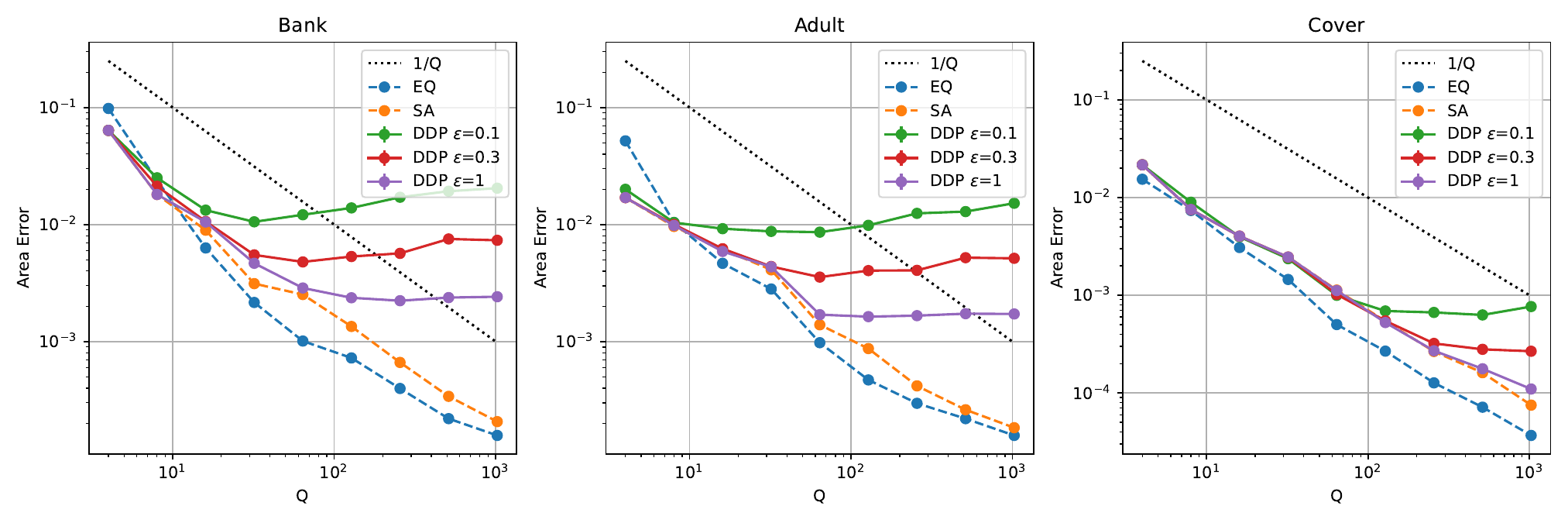}%
\caption{Effect of $\varepsilon$ (ROC, Logistic Regression).}%
\label{fig:epsilon-roc-lr}%
\Description[]{}%
\end{figure}
\begin{figure}[ht]
\centering
\includegraphics[width=\columnwidth]{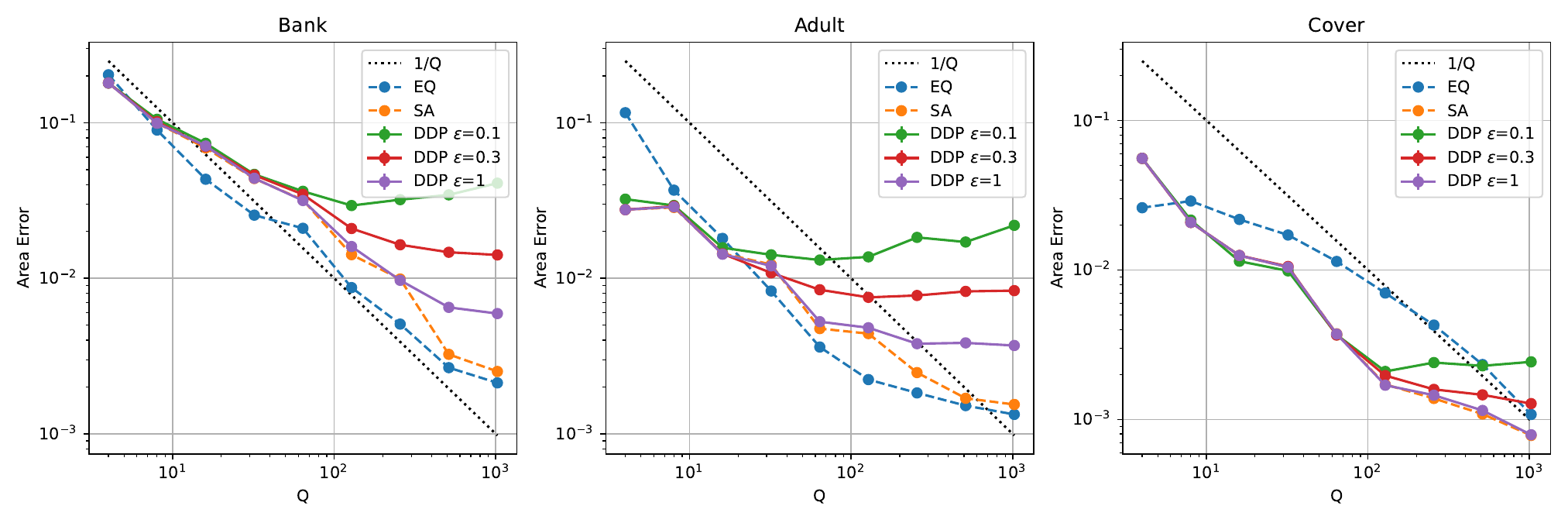}%
\caption{Effect of $\varepsilon$ (PR, Logistic Regression).}%
\label{fig:epsilon-pr-lr}%
\Description[]{}%
\end{figure}
\begin{figure}[ht]
\centering
\includegraphics[width=\columnwidth]{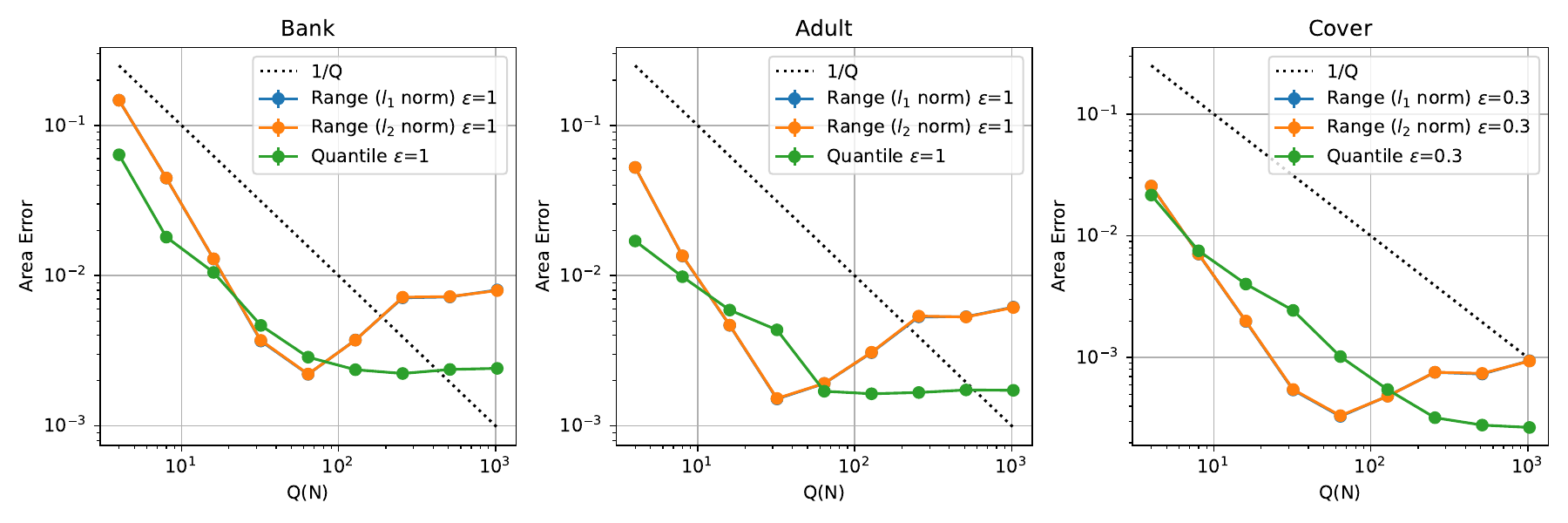}%
\caption{Comparison with range-query method (ROC, Logistic Regression).}%
\label{fig:range-roc-lr}%
\Description[]{}%
\end{figure}
\begin{figure}[ht]
\centering
\includegraphics[width=\columnwidth]{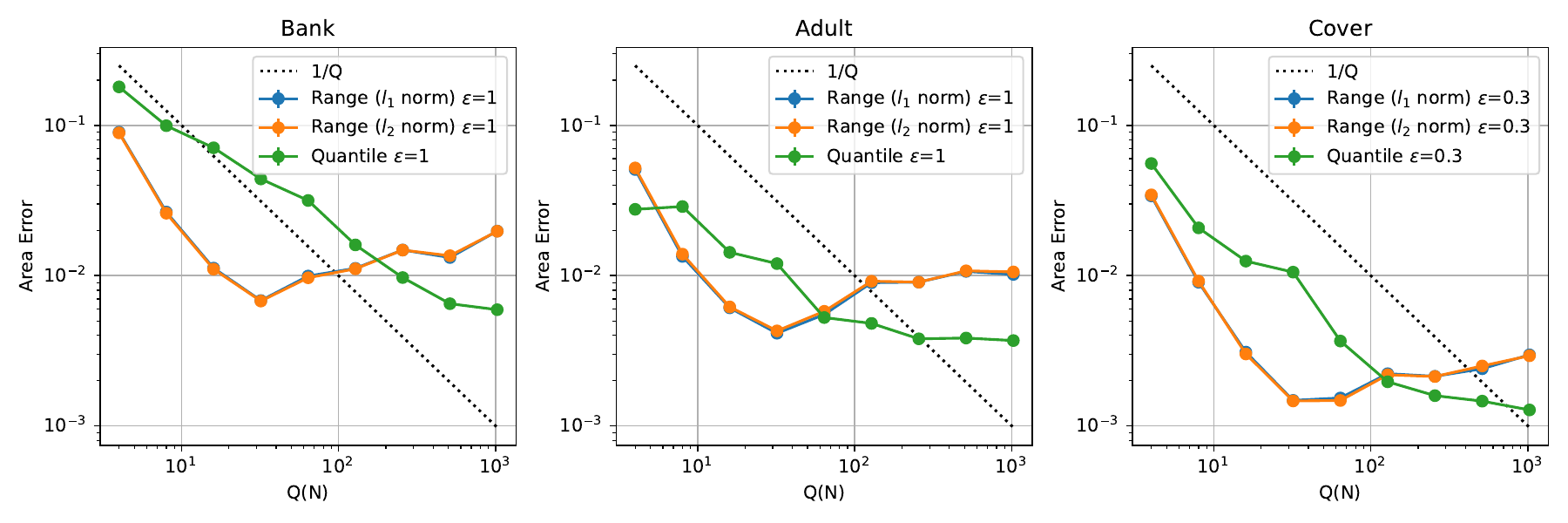}%
\caption{Comparison with range-query method (PR, Logistic Regression).}%
\label{fig:range-pr-lr}%
\Description[]{}%
\end{figure}
\begin{figure}[ht]
\centering
\includegraphics[width=\columnwidth]{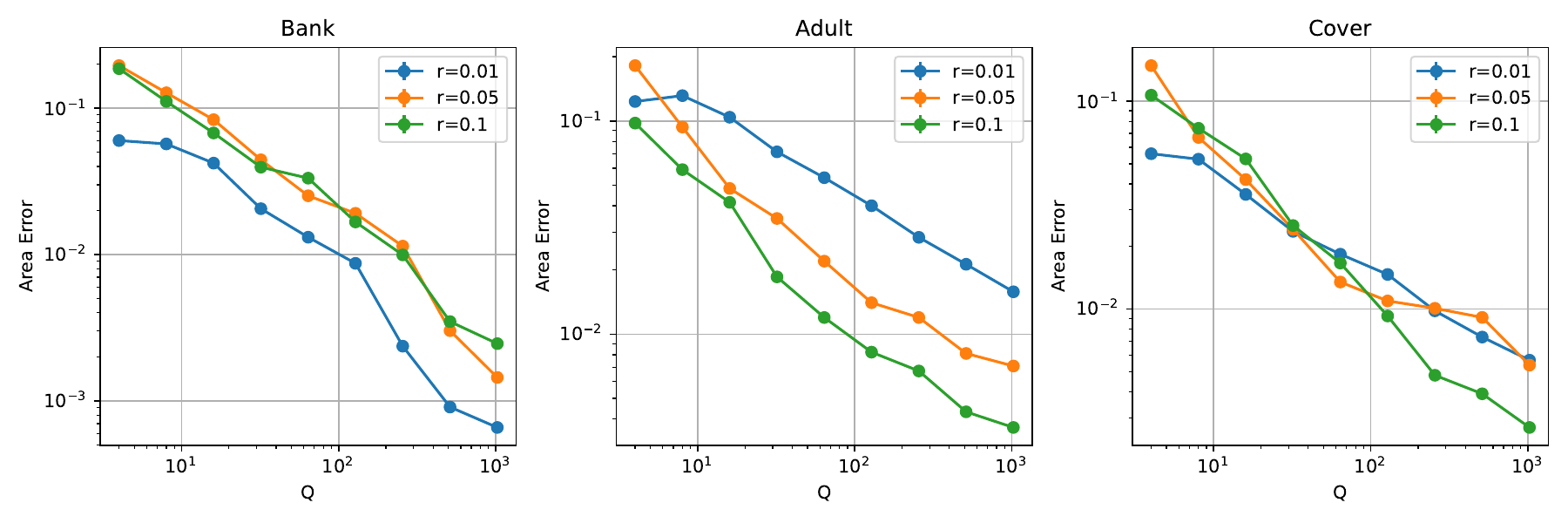}%
\caption{Effect of class imbalance on PR curve (Logistic Regression).}%
\label{fig:ratio-pr-lr}%
\Description[]{}%
\end{figure}